\documentclass[noinfoline]{imsart}

\usepackage[a4paper,top=3cm, bottom=3cm, left=2.5cm, right=2.5cm]{geometry}
\usepackage{fancyhdr}

\RequirePackage{amsthm,amsmath,amsfonts,amssymb}
\RequirePackage[numbers]{natbib}

\allowdisplaybreaks
\RequirePackage[colorlinks,citecolor=blue,urlcolor=blue]{hyperref}

\usepackage{mathrsfs}

\usepackage{comment}
\usepackage{caption}
\usepackage{mathtools}
\usepackage{float}
\usepackage{subcaption}
\usepackage{dsfont}
\usepackage{bbm}
\usepackage{amsmath}
\usepackage{algorithm}
\usepackage{eqparbox}
\usepackage{color, xcolor}
\usepackage{physics}
\usepackage{graphicx}
\graphicspath{{Images/}}
\usepackage{float}
\usepackage[figuresright]{rotating} 
\usepackage{subcaption}
\usepackage{pdfpages}
\usepackage{wrapfig}
\usepackage{tikz}
\usetikzlibrary{arrows.meta}

\usepackage{array}
\usepackage{booktabs}
\usepackage{enumerate}
\usepackage{enumitem} 
\usepackage{multicol}
\usepackage{algorithm}
\usepackage{algorithmic}

\newcommand{\N}{\mathcal{N}}
\newcommand{\Prob}{\mathbb{P}}
\newcommand{\E}{\mathbb{E}}
\newcommand{\btheta}{\boldsymbol{\theta}}
\newcommand{\bx}{\boldsymbol{x}}
\newcommand{\bX}{\boldsymbol{X}}
\newcommand{\bu}{\mathbf{u}}
\newcommand{\be}{\mathbf{e}}
\newcommand{\bv}{\mathbf{v}}
\newcommand{\diff}{\mathrm{d}}
\newcommand{\defeq}{\vcentcolon=}
\newcommand{\eqdef}{=\vcentcolon}

\allowdisplaybreaks

\usepackage[capitalize]{cleveref}

\theoremstyle{plain}
\newtheorem{theorem}{Theorem}[section]
\newtheorem{lemma}[theorem]{Lemma}

\newtheorem{assumption}[theorem]{Assumption}
\newtheorem{corollary}[theorem]{Corollary}

\newtheorem{proposition}[theorem]{Proposition}

\begin{document}

\pagestyle{fancy}
\fancyhead{}
\renewcommand{\headrulewidth}{0pt}
\fancyhead[CE]{Tao, Chandak and Kulkarni}
\fancyhead[CO]{Tao, Chandak and Kulkarni}

\begin{frontmatter}
\title{Fast convergence of a Federated Expectation-Maximization Algorithm}

  \begin{aug}
    \author[B]{\fnms{Zhixu}~\snm{Tao}\ead[label=e3]{zhixu.tao@princeton.edu}},
    \author[A]{\fnms{Rajita}~\snm{Chandak}\ead[label=e1]{rajita.chandak@epfl.ch}}
    \and
    \author[B]{\fnms{Sanjeev}~\snm{Kulkarni}\ead[label=e2]{kulkarni@princeton.edu}}
\address[A]{Institute of Mathematics,
Ecole Polytechnique Federale de Lausanne\printead[presep={,\ }]{e1}}

\address[B]{Department of Operations Research and Financial Engineering,
Princeton University\printead[presep={,\ }]{e2,e3}}
\end{aug}

\begin{abstract}
Data heterogeneity has been a long-standing bottleneck in studying the
convergence rates of Federated Learning algorithms. In order to better
understand the issue of data heterogeneity, we study the convergence rate of the
Expectation-Maximization (EM) algorithm for the Federated Mixture of $K$ Linear
Regressions model (FMLR). We completely characterize the convergence rate of the EM
algorithm under all regimes of number of clients and number of data points per
client, with partial limits in the number of clients. We show that
with a signal-to-noise-ratio (SNR) that is atleast of order
$\sqrt{K}$, the well-initialized EM algorithm converges to the ground truth
under all regimes. We perform experiments on synthetic data to
illustrate our results. In line with our theoretical findings,
the simulations show that rather than being a bottleneck, data heterogeneity can
accelerate the convergence of iterative federated algorithms.
\end{abstract}

\begin{keyword}[class=MSC]
  \kwd[Primary ]{62H12}
\kwd[; secondary ]{62H30}
\end{keyword}

\begin{keyword}
\kwd{Federated learning}
\kwd{EM Algorithm}
\kwd{Data Heterogeneity}
\kwd{Convergence rate}
\end{keyword}

\end{frontmatter}

\maketitle

\section{Introduction}
Leveraging increasingly large datasets for improved estimation accuracy is now
feasible in the digital age. However, curating such datasets presents
challenges, notably the high computational and storage costs, as well as
significant privacy concerns associated with centralizing personal data. In
order to resolve
these issues, recent machine learning efforts have been directed towards
distributed storage of data with a modified central processing system that can
still leverage the larger volume of data to provide more accurate estimation for
each individual client.
This field of study is referred to as Federated Learning (FL).
This approach is intended to not only preserve the privacy of the clients but
also to reduce the computational costs \citep{mcmahan2017communication}.

One fundamental challenge in the study of FL estimation
is the presence of non-independent and identically distributed (non-i.i.d.) data.
A common cause of non-i.i.d.\ data is that each client may have a different
underlying data generating process (DGP) \citep{ye2023heterogeneous} which can
correspond to differing ground-truth parameters. In other
words, if $P_j$ denotes the DGP for a client $j$, then $P_j\ne P_{j'}$ for clients $j\ne j'$.
This non-i.i.d.\ data renders many standard statistical models
inconsistent \citep{kairouz2021advances}.
The goal is then to accurately capture the heterogeneity in the data generating
process while maintaining a sufficiently rich function class. In the
classical parametric setting, one natural formulation of this comes in
the form of the mixture of linear regressions (MLR) model \citep{de1989mixtures,
faria2010fitting}. The standard
formulation of the MLR setup assumes some fixed $K$ (either known or unknown)
number of unique linear regressions in the mixture.
This reduces the problem to identifying $K$ distinct feature coefficient vectors.
To extend this to the FL setting wherein the heterogeneity is distributed across
clients, we assume that each client sees data from only one of the $K$ elements in
the mixture. Then, conditional on the mixture component, each
client has i.i.d.\ data points. This means that all the heterogeneity is
captured in the latent variable assigned to each client.

In the traditional centralized machine learning setting (which is equivalent to
centralizing all the data from the clients), the Expectation-Maximization (EM)
\citep{dempster1977maximum} algorithm has been one of the most successful methods
for studying MLR.
This leads us to the primary question:
\textit{Can a federated version of the EM algorithm consistently the federated
MLR model?}

\subsection{Our contributions}
The primary goal of this paper is to study the generalization of the EM algorithm
to the federated mixture of linear regressions.
To the best of our knowledge, this paper presents the first known results
statistical guarantees of the EM algorithm across different federated regimes for
mixtures of $K \geq 2$ linear regression. In presenting our main theoretical
results, we identify conditions under which EM converges faster in the federated
setting than in the centralized one with specific comparisons to existing rates
in the literature. Our results generalize the 2-mixture federated model
studied in \cite{reisizadeh2023mixture} under weaker assumptions. Moreover,
through refined analysis, we demonstrate that, contrary to common belief,
larger separation between mixture components does not always lead to better
convergence rates (see Theorems~\ref{thm:pop_cons} and \ref{thm:emp_consistency}).
Finally, we also highlight the regimes in which the algorithm
converges in a constant number of iterations (see Corollary~\ref{cor.1}).

The remainder of the paper is structured as follows: Section~\ref{sec:related
work} provides a detailed overview of related literature. Section~\ref{sec:setup} formalizes the
federated MLR model and details some key assumptions. Section~\ref{sec:results} presents the
main theoretical results. Section~\ref{sec:experiments} empirically evaluates
EM’s performance and the tightness of our theoretical assumptions. Finally, we
conclude in Section~\ref{sec:conclusion} with some proposals for future avenues of research.

\section{Related Work} \label{sec:related work}
\textbf{Data Heterogeneity:}
As mentioned earlier, non-i.i.d.\ data can limit the convergence rates
of classical FL algorithms \citep{li2019convergence,
khaled2020tighter, koloskova2020unified,
woodworth2020minibatch}. A growing body of work focuses on designing
optimization methods to accelerate convergence under non-i.i.d.\ data. Recent
advancements include alternative aggregation methods \citep{ye2023feddisco} and
regularization techniques \citep{kim2022multi, t2020personalized,
shoham2019overcoming, yao2020continual, li2021fedbn, xu2022fedcorr}. For
instance, \cite{tenison2022gradient} uses masking on gradients during the
averaging step to improve the rate of convergence.
SCAFFOLD \citep{karimireddy2020scaffold} employs variance reduction techniques
to mitigate drift caused by data heterogeneity. FedProx
\citep{li2020federated} incorporates a proximal term to constrain local updates
closer to the global model, while FedBN \citep{li2021fedbn} adds a batch
normalization layer to local models to address data heterogeneity.

Training a single global model by treating all datasets equally is often
inefficient. For example, in next-word prediction, clients may use different
languages \citep{hard2018federated}, making it essential to learn multiple local
models. Personalized Federated Learning (PFL) \citep{smith2017federated}
is a growing sub-field for addressing such problems.
In this vein, \cite{li2021ditto} optimizes both local and global models via a globally
regularized Multi-Task Learning framework, while \cite{fallah2020personalized}
applies a Model-Agnostic Meta-Learning approach for personalization. FedAMP
\citep{huang2021personalized} uses attentive message passing to encourage
collaboration among similar clients, enhancing personalization.
Clustered Federated Learning (CFL) \citep{ghosh2020efficient} is another
prominent framework for addressing this fundamental disparity in data from
different clients.
This approach groups clients into clusters, where each cluster shares a common
model. Additional methods include minimizing the
distance to the global model \citep{long2023multi}, weighted clustering
\citep{ma2022convergence}, and local gradient descent
\citep{werner2023provably}. \cite{mansour2020three} provide an empirical overview
of how personalized and clustered strategies perform in practice.
\\
\\
\textbf{Mixture Models and EM Algorithm:}
A common approach to modeling data heterogeneity in either the centralized or
federated setup is through treating the data-generating process as a mixture
model (see \cite{marfoq2021federated, su2022global} for various formulations
under different structural assumptions). While methods like the spectral approach
\citep{kannan2005spectral} and Markov Chain Monte Carlo (MCMC)
\citep{geweke2007interpretation} are sometimes used to analyze these models, the EM
algorithm \citep{dempster1977maximum} remains particularly popular among
practitioners due to its computational efficiency.

Recent advances in the literature have established convergence results for the EM algorithm
applied to mixtures of linear regressions (MLR) in the centralized setting
\citep{klusowski2019estimating, daskalakis2017ten,
kwon2020algorithm, 10.1214/19-EJS1660}.
\cite{yi2014alternating, yi2016solving} provide convergence guarantees
for noiseless MLR. \cite{10.1214/16-AOS1435} characterizes the local region
where EM converges to a statistically optimal point. \cite{kwon2019global}
proves the global convergence of EM for two-component MLR, and
\cite{kwon2020converges} provides result for a well-initialized EM for general
$K$-component MLR, both in the centralized setting.

In the federated setting, studies have examined the performance of EM under
compression \citep{dieuleveut2021federated}, highly specialized MLR models
(symmetric, two-component Gaussian components)
\citep{reisizadeh2023mixture, wu2023personalized}, and
with outliers using gradient descent \citep{tian2023unsupervised}. However, a
comprehensive theory of Federated MLR (FMLR) studied using the EM algorithm
remains an open question.

\section{Problem Setup and EM Algorithm}\label{sec:setup}
We start by describing the FMLR generation model. We will introduce additional
relevant notation in the following section.
\subsection{The FMLR model}
\label{sec:fmlr}
Suppose each of the $m$ clients has a latent variable $Z_j\in [K]$ and observes $n$
pairs of independent and identically distributed data points
$\{(\bX_i^j, Y_i^j)_{i=1}^n \}$ generated from the $Z_j$-th linear regression defined by
the parameter $\btheta_{Z_j}^*$.
This data generating process is described in Algorithm \ref{alg:fmlr}.
\begin{algorithm}[h]
\caption{The FMLR Algorithm}\label{alg:fmlr}
\textbf{Input}: K, m, n, and $\btheta_1^*, \ldots, \btheta_K^*$\\
\textbf{Output}: $\{\bX_i^j, Y_i^j\}_{i=1, j=1}^{i=n, j=m}$
\begin{algorithmic}[1] 
\FOR{j = 1, \ldots m}
\STATE Sample $Z_j\stackrel{\text{i.i.d.}}{\sim}\text{Unif}([K])$\COMMENT{latent
  variable, client ($m$) dependent}
\FOR{i = 1,\ldots n}
\STATE Sample
$\bX_i^j\stackrel{\text{i.i.d.}}{\sim}f_X$
    \COMMENT{predictor variables}\\
    Sample $\varepsilon_i^j\stackrel{\text{i.i.d.}}{\sim} f_{\varepsilon}$
    \COMMENT{noise}\\
   Generate
   $Y_i^j=\langle \bX_i^j, \btheta_{Z_j}^*\rangle+\varepsilon_i^j$
    \COMMENT{response variables}
\ENDFOR
\ENDFOR
\end{algorithmic}
\end{algorithm}
We note that this model inherently exhibits a clustered structure that
can be identified by grouping clients based on their latent variable $Z_j$.
Note that $\bX_i^j$ and $\varepsilon_i^j$ are independent of each other as well as
the latent variable $Z_j$ but $Y_i^j$ is not.
Furthermore, it is important to see that for each client $j$, there are $n$ pairs
of $\{\bX_i^j, Y_i^j\}_{i=1}^n$ sharing the same latent variable $Z_j$, which means
$\{\bX_i^j, Y_i^j, Z_j\}_{i=1}^n$ are not jointly i.i.d.

While there exist other formulations of data heterogeneity in FMLR modeling, we
restrict our work to this modelling scheme that focuses on data
heterogeneity caused by what is sometimes referred to in the literature as a
\textit{concept shift} \citep{kairouz2021advances},
where $P_j(x, y) \ne P_{j'}(x, y)$ for $j\ne j'$ arises from
$P_j(y|x) \ne P_{j'}(y|x)$
even if $P_j(x)$ is the same for all $j$. This can be understood in the
context of user preferences.
For example, when presented with identical collection of items,
different users may label items differently based on personal preferences that can be
categorized based on more general features like regional or demographic variations.

\subsection{Notation}\label{notation}
In this section we collect some notation that help in formulating our main
results in the next section.
\begin{itemize}
\item $d$: the dimensionality of the problem (i.e.\ number of features or covariates), known and fixed.
\item $\bX \in \mathcal{X} \subseteq \mathbb{R}^d$: collection of features
  (or covariates).
\item $Y \in  \mathcal{Y} \subseteq\mathbb{R}$: response variable.
\item $Z \in \{1, \ldots, K\} \defeq\mathcal{Z}$: latent (unobserved) variable
indicating the element of the mixture, uniformly distributed.
\item $K$: number of mixture components, known and fixed.
\item $m$: number of clients.
\item $n$: number of data points per client.
\end{itemize}
We use the set notation $[n] = \{1,\dots, n\}$ and therefore $\bX_{[n]} = \{\bX_1,\dots, \bX_n\}$.
The index $j\in [m]$ identifies the client while the index $i \in [n]$ denotes the observation. Moreover,
$f_{\btheta}(\cdot)$ denotes the probability density function of a continuous
(possibly multivariate) random variable with parameter $\btheta$, and
$g_{\btheta}(\cdot)$ denotes the probability mass function of a discrete random
variable with parameter $\btheta$.
We use $\|\cdot\|$ to denote the Euclidean norm.

Let $\btheta_k^*$ be the $k$-th ground truth coefficient vector for $k \in [K]$.
In our one-step analysis, we use
$\btheta_k$ and $\btheta_k^+$ to denote the current and the next estimates of
$\btheta^*_k$, respectively.
Empirical (data-dependent) estimates are denoted by $\widehat\btheta_k$ and
$\widehat\btheta_k^+$. Define the maximum and minimum separations between the
true coefficient vectors as
\[\Delta_{\max} :=\max_{k\ne k'}\|\btheta_k^*-\btheta_{k'}^*\|
  \quad\text{and}\quad
  \Delta_{\min}:=\min_{k\ne k'}\|\btheta_k^*-\btheta_{k'}^*\|,\]
respectively.
The signal-to-noise ratio (SNR) is given by
$\Delta_{\min}/\sigma$, where $\sigma$ is the variance of the noise. Moreover,
define $\mathbb{E}_{k}[\cdot]$, as the expectation with
respect to the joint distribution of $(\bX, Y)$ conditional on $Z=k$.
That is, $ \mathbb{E}_{k}[\cdot] = \mathbb{E}[\cdot \mid Z = k]$.
Finally, for two sequences $a_n$ and $b_n$, we write
$a_n = O(b_n)$ if $\limsup_{n \to \infty} \frac{a_n}{b_n} \leq c$, for some
constant $c >0$.

\subsection{EM Algorithm}
We present the EM algorithm specifically in the context of FMLR models.
For an overview of the EM algorithm in the classical (or centralized) setting see
\cite{dempster1977maximum, meng1997algorithm}.

We start by assuming the data generating process as described in
Algorithm~\ref{alg:fmlr}.
To estimate the parameters $\{\btheta_k^*\}_{k=1}^K$ in the presence of latent
variables, the EM algorithm approximates the MLE:
\begin{align}
  \label{eq:loglikelihood}
  \ell_m(\btheta) =
  \frac{1}{m}\sum_{j=1}^m \log
  \int_{\mathcal{Z}}f_{\btheta}(\bX^j_{[n]}, Y^j_{[n]},
  z_j)\diff z_j ,
\end{align}
which is not only typically a non-concave function, but also depends on
the unobserved latent variables, $z_j$ and so, in general, is intractable.
In order to bypass this dependency, the algorithm lower bounds
the log-likelihood defined by the following function:
\begin{align}
  \label{eq:qm}
  Q_m(\btheta|\widehat{\btheta}^{(t)})
  &= \frac{1}{m}\sum_{j=1}^m\int_{\mathcal{Z}}
    g_{\widehat{\btheta}^{(t)}}(z_j|\bX_{[n]}^j, Y_{[n]}^j)
    \log f_{\btheta}(\bX_{[n]}^j, Y_{[n]}^j, z_j)
    \diff z_j,
\end{align}
where $g_{\btheta'}(z|\bx_{[n]}, y_{[n]})$ denotes the conditional
probability mass function of $z$ conditional on $(\bx_{[n]}, y_{[n]})$ and
$\widehat{\btheta}^{(t)} = [\widehat{\btheta}^{(t)}_1, \ldots,
\widehat{\btheta}^{(t)}_K]$ is an estimate of the true parameters
The construction of  $Q_m$ is referred to as the E-step, since it removes
dependency on the latent variable, $Z$ by taking an expectation over it.
The EM algorithm then generates a new estimate for the parameter by maximizing
the approximation to the likelihood, $Q(\btheta|\widehat{\btheta}^{(t)})$ with respect to
$\btheta$. That is, the subsequent estimator produced by the algorithm given an
initial estimator $\widehat{\btheta}^{(t)}$ is defined as
\[
  \widehat{\btheta}^{(t+1)}
  = \arg\max_{\btheta\subset\Theta}Q_m(\btheta|\widehat{\btheta}^{( t )}).
\]
This is referred to as the M-step.
We note here that this setup trivially works for when each client has a
different number of data points, $n_m$, by defining $n = \min_m n_m$. Although,
in other generalizations of the DGP where the probability distribution of the
latent variable is non-uniform, it may be informative to use the varying number
of samples for each client.
In the above construction of the algorithm, we assume finite $m$ and $n$, which
we will refer to as the empirical algorithm.
However, for theoretical purposes it is helpful to consider the limiting
quantities (either with respect to $m$, $n$ or both). We refer to this as the
population version of the EM algorithm.
For our purposes it is interesting to consider the population quantity
with respect to the limit $m \to \infty$ only. The reason for this is that under
the $n \to \infty$ limit each client can be treated independently as a standard
estimation problem, removing the need for any federated approach.
We highlight that the population EM algorithm assumes that we have access to the
joint distribution $f_{\btheta^*} (\bx, y)$.
In particular, we can write down the population analog of $Q_m$, denoted by $Q$
as
\begin{align}
  \label{eq:pop_Q}
  Q(\btheta|\btheta^{(t)})
  &= \int_{\mathcal{X}\times
    \mathcal{Y}}
    \left(\int_{\mathcal{Z}} g_{\btheta^{(t)}}(z|\bx_{[n]}, y_{[n]})
    \log f_{\btheta}(\bx_{[n]}, y_{[n]}, z)\diff z \right)
    f_{\btheta^*}(\bx_{[n]}, y_{[n]})\diff \bx_{[n]}\diff y_{[n]}.
\end{align}
Without any further information on the distributions of any of the random
variables $\bX, Y$ or $Z$, we would stop here and any theoretical guarantees on
the algorithm would have to directly analyse either the $Q$ or $Q_m$ functions.
In pratice, it is near-impossible to get anything informative regarding the sequence
of parameter estimates $\{\widehat{\btheta}^{(t)}\}_{t\geq 1}$ in this minimal
assumption regime. For most pratical purposes it is helpful to place some
assumptions on the data generating model. In particular, we will choose to
operate under the standard Gaussian model.
\begin{assumption}[DGP]
  \label{as:dgp}
Let $\bX \sim \N(0, I_d)$ and $\varepsilon \sim \N(0, \sigma^2),$
where $\sigma > 0$ is a constant. Furthermore, $\bX \perp \varepsilon$.
\end{assumption}
We can now simplify the two steps of, both, the population and empirical EM
iterations, starting with the population EM.
\begin{proposition}[Population EM]
  \label{prop:population_EM}
  Suppose Assumption~\ref{as:dgp} holds and
  $\{(\bX_i, Y_i)\}_{i=1}^n$ are generated by Algorithm~\ref{alg:fmlr} with $m =
  \infty$.
  Then one iteration of the population EM, given the current estimates
  $\btheta_k, k \in [K]$, is given by
  \begin{align*}
    &\text{E-Step:}\quad w_k(\btheta) =
      \frac{\exp(-\frac{1}{2\sigma^2}
      \sum_{i=1}^n(Y_i-\langle \bX_i,\btheta_k\rangle)^2)}
      {\sum_{l=1}^K\exp(-\frac{1}{2\sigma^2}
      \sum_{i=1}^n(Y_i-\langle \bX_i,\btheta_l\rangle)^2)}
    && \forall k \in [K],
    \\
    &
      \text{M-Step:}\quad
      \btheta^+_k
      =
      \mathbb{E}\left[w_k(\btheta) \sum_{i=1}^n\bX_i\bX_i^T\right]^{-1}
      \mathbb{E}\left[w_k(\btheta)\sum_{i=1}^nY_i\bX_i^T\right]
    && \forall k \in [K].
  \end{align*}
\end{proposition}
The proof of this proposition is deferred to
the Appendix. See Appendix~\ref{sec:proofs_propositions} for the proof all
results in this section.

\begin{proposition}[Empirical EM]
  \label{prop:empirical_EM}
  Suppose Assumption~\ref{as:dgp} holds and $\{(\bX_i^j, Y_i^j)\}_{i=1, j=1}^{i=n, j=m}$
  are generated by Algorithm~\ref{alg:fmlr}.
  Then one iteration of the empirical EM, given the current
  estimates $\widehat{\btheta}_k, k \in [K]$, is given by
  \begin{align*}
    &\text{E-Step:}\quad
      w_k^j(\widehat\btheta)
      = \frac{\exp(-\frac{1}{2\sigma^2}
      \sum_{i=1}^n(Y_i^j-\langle \bX_i^j,\widehat\btheta_k\rangle)^2)}
      {\sum_{l=1}^K\exp(-\frac{1}{2\sigma^2}
      \sum_{i=1}^n(Y_i^j-\langle \bX_i^j,\widehat\btheta_l\rangle)^2)}
    && \forall k \in [K],
    \\
    &\text{M-Step:}\quad
      \widehat{\btheta}^+_k
      = \left(\sum_{j=1}^mw_k^j(\widehat\btheta) \sum_{i=1}^n\bX_i^j\bX_i^{jT}\right)^{-1}
      \left[ \sum_{j=1}^mw_k^j(\widehat\btheta)\sum_{i=1}^nY_i^j\bX_i^j \right]
    && \forall k \in [K].
  \end{align*}
\end{proposition}

\section{Main Results}\label{sec:results}
We are now ready to present our main theoretical result.
It is natural to break this up into two distinct statements, one
for the population EM and one for the empirical EM.
We start by making an assumption on the initialization of the algorithm that
ensures identifiability of the solution.
\begin{assumption}[Identifiability]
  \label{as:identifiability}
  The initial estimates, $\{\widehat{\btheta}^{(0)}_k: k \in [K]\}$, are chosen such that
  \begin{align*}
    \|\widehat{\btheta}^{(0)}_k-\btheta_k^*\|
    \leq
    \alpha \Delta_{\min}
    \ \forall \
    k \in [K]
  \end{align*}
  where $\alpha \in (0, 1/4)$ is a constant.
  Furthermore, $\widehat{\btheta}^{(0)}_k = \btheta^{(0)}_k$ for all $k \in [K]$.
\end{assumption}

This type of assumption is very common in the literature of mixture models,
albeit with different range of values permitted for $\alpha$ (which varies
depending on the other assumptions of the model).
By ensuring the initializations are closest (in euclidean distance) to
a single true component, the initialized model is well-defined and so are the
correspoding iterates of the algorithm. It guarantees, in essence, that a single
initialization cannot converge to two different ground truth vectors.

We now state the uniform convergence result for the population EM.
\begin{theorem}[Uniform consistency]
  \label{thm:pop_cons}
  Suppose Assumptions~\ref{as:dgp} and~\ref{as:identifiability} hold.
  If $\text{SNR} \gtrsim \sqrt{K}$,
  then the estimates generated after one iteration of the Population EM
  algorithm (as defined in Proposition~\ref{prop:population_EM}) satisfy
  \begin{align*}
    \max_{k\in [K]}\|\btheta_k^+-\btheta_k^*\|
    &\lesssim
      \frac{\alpha \Delta_{\min}\sqrt{n}\sigma e^{-C_\alpha n}}{1-Ke^{-n}}
      +\frac{\Delta_{\max}e^{-n/K^2}}{\sqrt{n}(1-Ke^{-n})}
    +\frac{e^{-n}}{\sqrt{n}}.
  \end{align*}
  where $C_{\alpha}= \frac{(1-4\alpha)^2}{64 \alpha^2}$.
\end{theorem}
The proof of this theorem and all othe results in this section are provided in
Appendix~\ref{appendix:proofs_theorems}. Additional details with regards to the
rates are included in the Appendix. The interested reader may consider the
details of additional technical results in Appendix~\ref{appendix:technical},
which are used heavily in the proofs of the main theorems.

From Theorem~\ref{thm:pop_cons}, we can see that provided we start with a
relatively good initialization, conditional on $\Delta_{\max}$ and
$\Delta_{\min}$ being well-controlled, one step of the population EM will
converge to the true parameters.
This explicit dependency of the error on the magnitude (as defined by
$\Delta_{\min}$ and $\Delta_{\max}$) of the problem is possibly counter-intuitive.
Most literature on cluster identification makes the assumption that the larger
the distance between clusters, the easier it is for iterative algorithms like EM
to identify the true cluster centers \citep{10.1214/16-AOS1435,
  kwon2020converges, kwon2019global}, and thus this
quantity is not typically explicitly captured in the error bounds.
Our result shows that, in the case of federated EM, prohibitively large maximal
distances between two clusters actually implies a
larger $l_2$ error. We conjecture this is due to the fact that in identifying
the correct centers, individual center-level accuracy is sacrificed in some
sense for worst-case error due to the partial dependency structure of the data.
This hypothesis is verified and discussed further with simulations in
Section~\ref{sec:experiments}.

In order to complete our analysis of the one-step federated EM algorithm, we now
present the convergence of the empirical EM algorithm.
\begin{theorem}[Empirical uniform consistency]
  \label{thm:emp_consistency}
  Suppose Assumptions~\ref{as:dgp} and~\ref{as:identifiability} hold.
  Furthermore, assume the following constraints on the model parameters:
  \begin{enumerate}
  \item $n \gtrsim \log(K)$,
  \item $m \gtrsim K\log(K)$,
  \item $\text{SNR} \gtrsim \sqrt{K}$ and,
  \end{enumerate}
  If we define
  $D_t \defeq \max_{k\in [K]}\|\widehat{\btheta}^{(t)}_k-\btheta_k^*\|\le\alpha\Delta_{\min}$
  as the worst-case error of the current empirical iterate.
  Then, with probability at least $1-3\delta/K^2$, the estimates generated after
  one iteration of the empirical EM algorithm (see Proposition~\ref{prop:empirical_EM})
  is controlled by
  \begin{align*}
    &\quad\max_{k\in [K]}\|\widehat{\btheta}_k^{(t+1)}-\btheta^*_k\|
    \precsim
    \begin{cases}
      \frac{D_t}{mn^{1/4}}
      +\frac{\Delta_{\max}}{m\sqrt{n}}
      + (n^{3/2}\Delta_{\min} + n\Delta_{\max})e^{-n}
      &
        \text{if }\ m \lesssim \exp(n)
      \\
       \frac{KD_t}{n^{1/4}}e^{-(C_\alpha-1) n/2} + K\sigma\sqrt{\frac{d}{n}}e^{-n} + \frac{e^{-n}}{n^{1/4}}
      &
        \text{if }\  m \gtrsim \exp(n)
    \end{cases}
  \end{align*}
\end{theorem}
As we see in the statement of the theorem, the precise rate of convergence
depends on the relationship between the two key variables $m$ and $n$. The error
bound consists of two parts: the approximation error that comes from analyzing
$\|\widehat\btheta_k^+-\btheta_k^+\|$ and the generalization error that comes
from $\|\btheta_k^+-\btheta_k^*\|$.
If we ignore the additional parameters like $K, d$ and $\sigma$, that under our
settings are assumed to be constants, the approximation error is the leading
term in the rate (see proof in Appendix \ref{proof:empirical_thm} for a complete
expression of the rate including dependency on all other parameters). However,
when $m$ is sufficiently large (on the order of eponential in $n$), the
approximation error is overtaken by the population error. This result is
consistent with many of the standard results dealing with convergence of
estimators. The key difference here is the fact that $n$ is always treated as a
finite constant and as such the population error could in fact contribute to the
total error in a non-trivial manner if $n$ is small and $m$ is relatively large.

Theorem \ref{thm:emp_consistency} also shows how the maximum
separation $\Delta_{\max}$
affects the convergence rate depending on the magnitude of $m$ and $n$. Unlike
existing literature, which identifies $\Delta_{\max}$ in a restricted regime
(i.e.\ specific range of $n$ or centralized EM)
\citep{10.1214/16-AOS1435, klusowski2019estimating, kwon2020converges, 10.1214/19-EJS1660,
yan2017convergence, reisizadeh2023mixture}, we have accounted for the role of
$\Delta_{\max}$ across all regimes. When $m$ grows no faster than exponentially
in $n$, the effect of $\Delta_{\max}$ is controlled by $m\sqrt{n}$. Thus, a
prohibitively fast growth of $\Delta_{\max}$ can actually lead to the federated
EM algorithm converging to the wrong mixture model, and a more careful
application of the EM algorithm would be required to verify whether the error
can be made to vanish.

The following corollary highlights the number of iterations that will be
required under the assumptions of Theorem~\ref{thm:emp_consistency} to enure a loss of at
most $\varepsilon$.
\begin{corollary}\label{cor.1}
  Suppose the assumptions from Theorem~\ref{thm:emp_consistency} hold.
  In addition, if $m \lesssim \exp(n)$
  and $\frac{\Delta_{\max}}{m\sqrt{n}}\leq \varepsilon/2$ then,
  $\max_{k\in [K]}\|\widehat\btheta^{(T)}_k-\btheta_k^*\|\le \varepsilon$
  for
  \begin{align*}
    T\geq
    \frac{2\log(\frac{\alpha \Delta_{\min}}{\varepsilon})}
    {\log(mn^{1/4})}
  \end{align*}
  with probability $(1-3\delta/K^2)^T$ for any $\varepsilon > 0$.
  \\
  When $m \gtrsim \exp(n)$,
  $\max_{k\in [K]}\|\widehat\btheta^{(T)}_k-\btheta_k^*\|\le \varepsilon$
  for
  \begin{align*}
    T \geq \frac{\log(\frac{2\alpha\Delta_{\min}}{\varepsilon})}{n
    +\frac{1}{4}\log n - \log K} = O(1)
  \end{align*}
  with probability $(1-3\delta/K^2)^T$ for any $varepsilon>0$.
\end{corollary}
Compared to the classical EM algorithm, federated EM achieves faster convergence
in certain regimes. In particular, note that for $m$ and $n$ sufficiently large,
Corollary~\ref{cor.1} implies a constant number of convergence, while
in the classical setting, previous results have required a growing number of
iterations. For example \cite{kwon2020converges} establish a linear dependency of
the number of iterations $T$ on $n$. \cite{kwon2019global}, for the specialized
$K=2$ case, and \cite{reisizadeh2023mixture} in the specialized $K=2$ case for
the federated model both require $T$ to grow logarithmically in $n$ (in $mn$ for
the federated setting) for convergence. We note that here
\cite{reisizadeh2023mixture} do show convergence in
constant number of iterations for the specialized $K=2$ model under stronger
assumptions on the relationship between $m$ and $n$ as well as on
initialization. We generalize their results on both fronts and extend the
theory to a general number of mixtures. We conjecture from our analysis that this
phenomenon occurs because data points on the same client share the same latent
variable, eliminating the need to identify the cluster membership of each
individual data point once the latent variable of a client has been determined.
Consequently, the clustering task becomes easier and more efficient.

\section{Experiments}\label{sec:experiments}
In this section, we evaluate the performance of the federated EM
algorithm using simulated datasets that satisfy the
assumptions for which we have established theoretical results.
In Figures \ref{fig:experiments_with_n}-\ref{fig:exp_with_max}, the
left subplot shows the average maximum error
($\max_{k\in[K]}\|\btheta_k^T-\btheta_k^*\|$)
over 100 repititions and the right subplot
shows the average number of iterations required to converge over 100
repititions with respect to the number of clients $m$.
For each experiment, we randomly initialize $\{\btheta_k\}_{k=1}^K$
to satisfy Assumption~\ref{as:identifiability} with $\alpha=1/5$
and we set $\sigma = 1$ for simplicity.
For a complete description of all parameters used in each simulation, we refer
the reader to Appendix~\ref{experiment details}.
Furthermore, all replication files can be found on
\href{https://github.com/rajitachandak/FederatedEM\_replication}{Github}.

We begin by examining the effect of the number of data points $n$ that each
client holds on the convergence rate
in Figures \ref{fig:exp_with_small_n} and \ref{fig:exp_with_large_n}.
Figure~\ref{fig:exp_with_small_n} shows how the EM algorithm behaves when $m$ grows
at least polynomially in $n$, while Figure \ref{fig:exp_with_large_n} shows the
behavior when $m$ is independent of $n$. In both cases,
the algorithm converges to the ground truth after a near-constant number of iterations.
The key takeaway is that the EM algorithm performs well in both cross-silo
(small $m$, large $n$, e.g., few companies with lots of data),
and cross-device FL (small $n$, large $n$ e.g., millions of mobile devices with
few data points).
\begin{figure}[h]
    \begin{subfigure}{.5\textwidth}
    \centering
    \includegraphics[width=\textwidth]{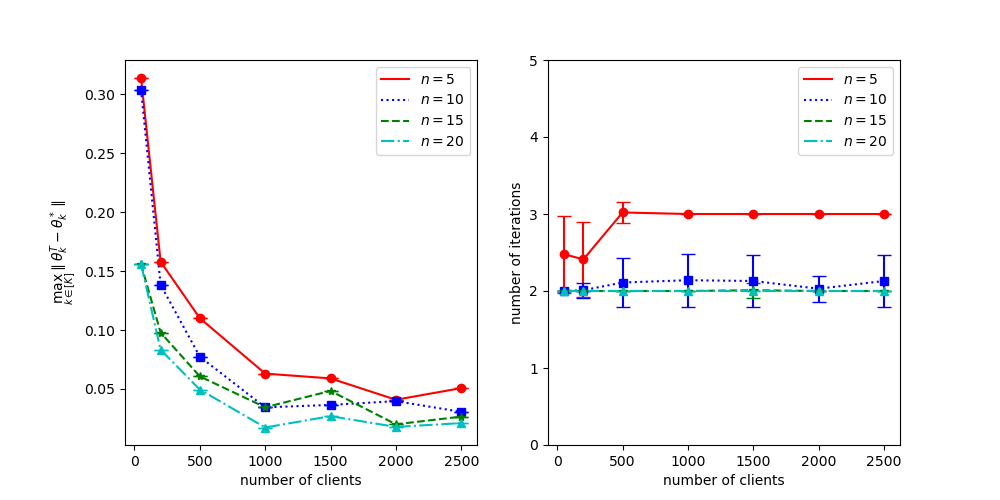}
    \caption{Effect of small $n$}
    \label{fig:exp_with_small_n}
    \end{subfigure}
    \begin{subfigure}{.5\textwidth}
    \centering
    \includegraphics[width = 1.0\textwidth]{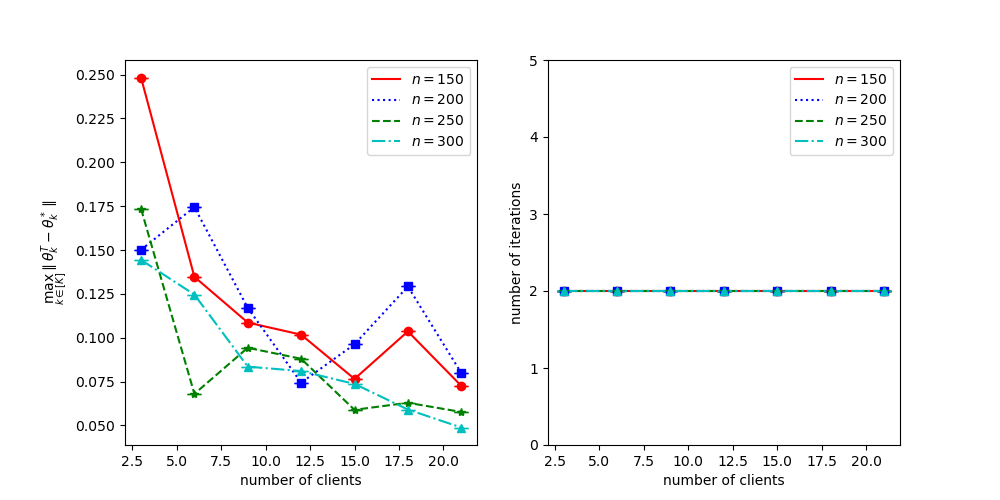}
    \caption{Effect of large $n$}
    \label{fig:exp_with_large_n}
    \end{subfigure}
\caption{Effect of number of data points $n$}
\label{fig:experiments_with_n}
\end{figure}
Figure \ref{fig:exp_with_K} shows the effect of number of clusters $K$ on the
convergence rate. We notice here that when the number of components in the
mixture model increases, the algorithm generally requires more iterations to
converge however the growth in the number of iterations is not even
polynomial with respect to the number of clusters, which is an important
consideration for the scalability of the algorithm. This observation aligns with
our theoretical findings (see Appendix~\ref{proof:empirical_thm} for details).

Figure \ref{fig:exp_with_d} shows the effect of dimensionality $d$ on the
convergence rate. We see that the average maximum error increases with $d$ over
$m$. Furthermore, we observe that higher dimensionality increases the number of
iterations required for convergence generally. However, it is unclear from the
simulations and the theory as to whether the dependecny observed is optimal in
any sense. The high-dimensional properties (i.e., when
$ d \propto n$) of EM remains an open question, even in the centralized setting.
\begin{figure}[h]
  \begin{minipage}[c]{0.5\linewidth}
    \centering
    \includegraphics[width = 1.0\textwidth]{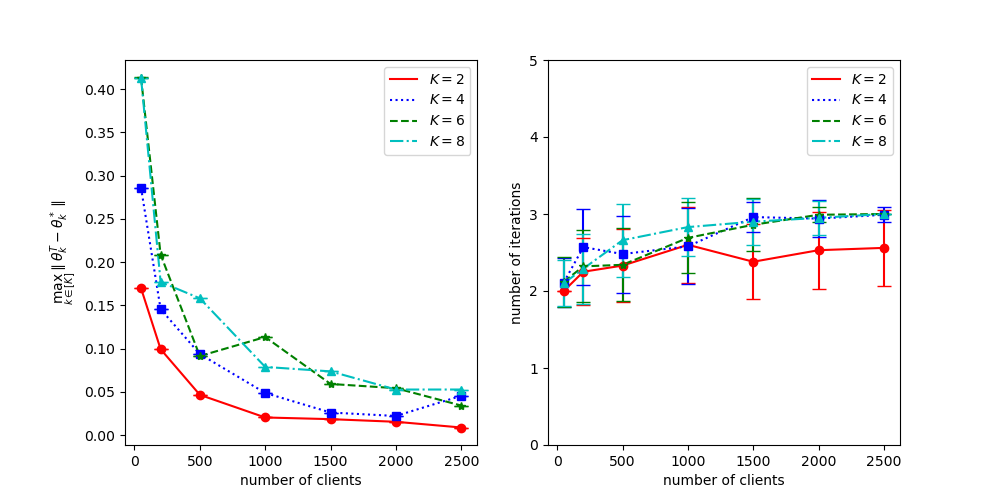}
    \caption{Effect of number of clusters $K$}
    \label{fig:exp_with_K}
  \end{minipage}
  \hfill
  \begin{minipage}[c]{0.5\linewidth}
    \centering
    \includegraphics[width = 1.0\textwidth]{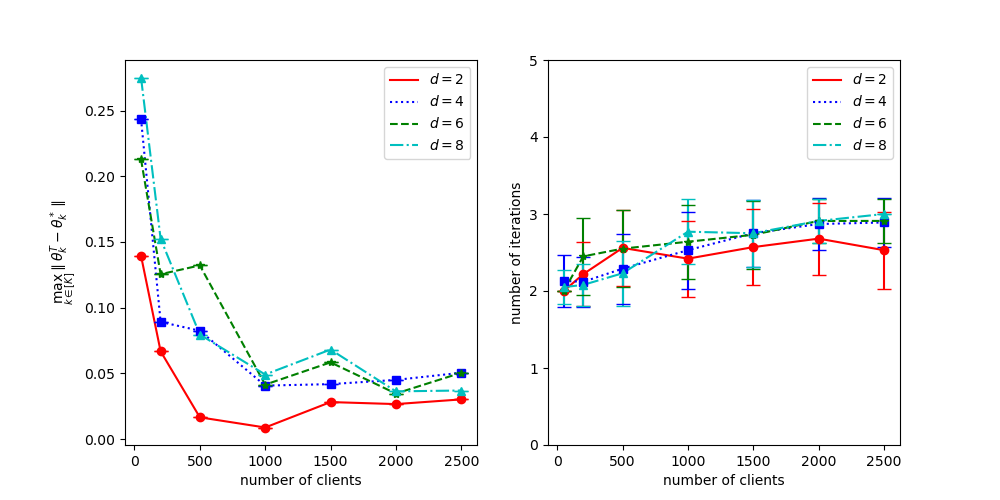}
    \caption{Effect of dimension $d$}
    \label{fig:exp_with_d}
  \end{minipage}
\end{figure}

Figure \ref{fig:exp_with_SNR} shows the effect of SNR on the convergence rate.
As the SNR increases, the algorithm appear to converge faster with
smaller Euclidean error. It is also worth noting Theorems
\ref{thm:pop_cons} and \ref{thm:emp_consistency} suggest that a lower bound of SNR for
identifiability of the solution should be given by $\sqrt{K}$
which in our simulations for $K=3$, Figure~\ref{fig:exp_with_SNR} shows that
when the SNR is less than $\sqrt{3}$, the algorithm requires significantly more
iterations to converge. The error of the converged iterates also seems to depend
on the SNR. It remains unclear whether the bound on SNR found in our theory is
the tightest possible bound.
\begin{figure}
  \begin{minipage}[c]{0.5\linewidth}
    \centering
    \includegraphics[width = \textwidth]{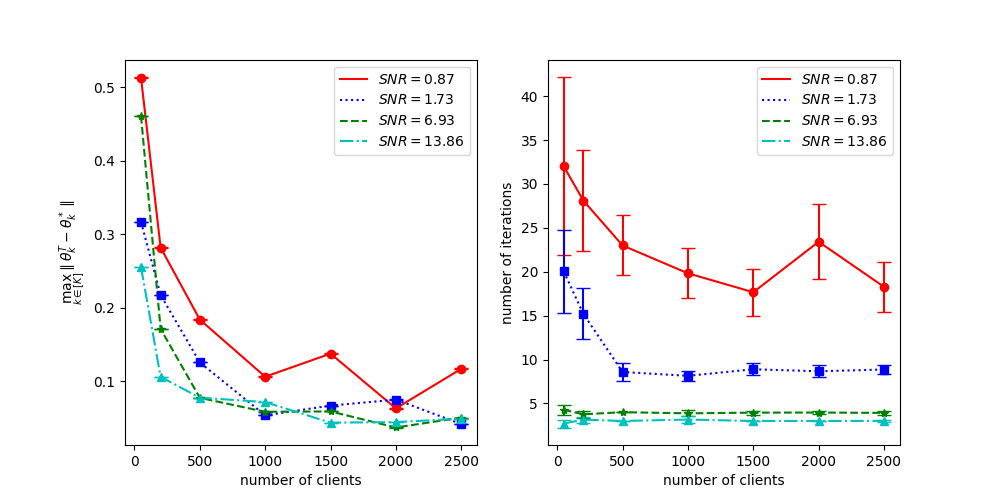}
    \caption{Effect of SNR}
    \label{fig:exp_with_SNR}
  \end{minipage}
  \hfill
  \begin{minipage}[c]{0.5\linewidth}
    \centering
    \includegraphics[width = \textwidth]{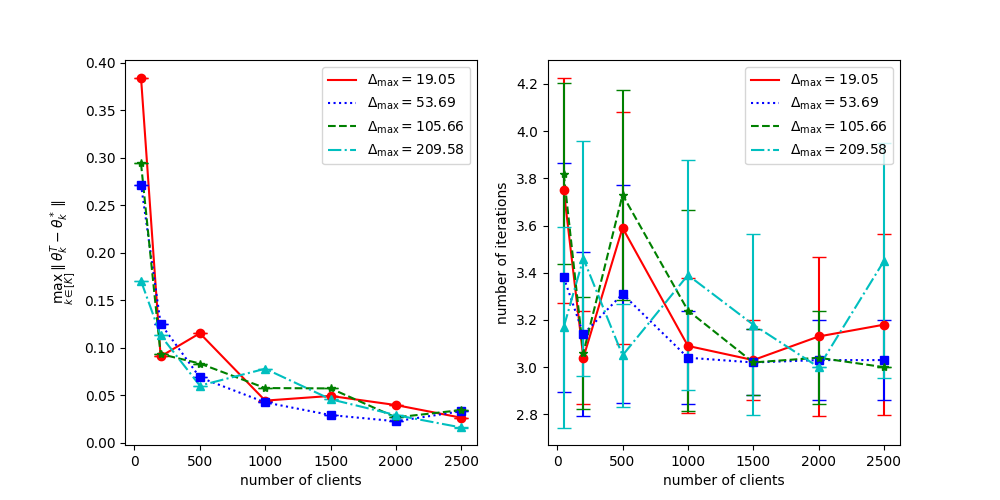}
    \caption{Effect of $\Delta_{\max}$}
    \label{fig:exp_with_max}
  \end{minipage}
\end{figure}

Finally, Figure~\ref{fig:exp_with_max} shows the effect of the maximum separation
$\Delta_{\max}$. Notably,
a larger $\Delta_{\max}$ does not necessarily guarantee a faster
convergence or uniformly lower error. In
fact, in some of the simulations a smaller $\Delta_{\max}$ corresponds to
smaller errors or fewer iterations. This aligns with the bounds derived in
Section~\ref{sec:results} and challenges the commonly held belief in the
literature that greater cluster separation always improves the convergence of
iterative algorithms, even when the number of clusters is small.

\section{Conclusions and future work}\label{sec:conclusion}
This paper provides the first known convergence rates for the EM algorithm under
all regimes of $m$ and $n$ in Federated Learning. The key findings show that
when the data heterogeneity among clients can be described by the FMLR model,
the well-initialized federated EM algorithm can find the true regression
coefficients in only a constant number of iterations.
This paper also provides theoretical and experimental results to challenge the
commonly held belief that greater separation in clusters of data is always
beneficial to the EM algorithm.
We conclude with the discussion of some avenues for future work.
\begin{itemize}
\item \textbf{Parameter dependencies:} While the results presented here show
  relatively weak set of assumptions on parameters like the SNR, it may be
  worth exploring minmax dependencies within the federated learning framework,
  which remains an open question even outside the mixture of linear regression
  modeling setup.
  \item \textbf{Restricted communication:} A common constraint in practical
    cases of federated learning deal with restricting the amount of
    communication between clients and a central server. It would be of interest
    to medical and financial applications to generalize the existing results
    under communication restriction/loss regimes.
  \item \textbf{Generalizing mixture models:} Within both the
    federated and classical learning setups it is of interest to work with more
    general distributions, in particular ones that deal with heavier tails than
    Gaussian densities or have a restricted support.
\end{itemize}

\begin{appendix}

\section{Proofs for Section~\ref{sec:setup}}
\label{sec:proofs_propositions}
In this section, we will prove the two propositions from Section~\ref{sec:setup}.
Recall that we denote $f_{\btheta}(\cdot) $ as the probability density function
of a continuous random variable and $g_{\btheta}(\cdot)$ as the probability mass
function of a discrete random variable with parameter(s) $\btheta$.

\subsection{Proof of Proposition \ref{prop:population_EM}}
\label{pf:prop_empEM}
\begin{proof}
Recall that the joint density of $(\bX_{[n]}, Y_{[n]}, Z)$ can be written as
\begin{align*}
  f_{\btheta}(\bX_{[n]}, Y_{[n]}, Z)
  &= \Prob(Z)f_{\btheta}(\bX_{[n]}, Y_{[n]}|Z)
  \\
  &
  =\frac{1}{K}f(\bX_{[n]})\prod_{i=1}^n\mathcal{N}(\langle \bX_i, \btheta_Z\rangle, \sigma^2)
  \\
  &=\frac{1}{K}f(\bX_{[n]})
    \exp\left\{-\frac{1}{2\sigma^2}
    \sum_{i=1}^n(Y_i-\langle \bX_i, \btheta_Z\rangle)^2
    \right\},
\end{align*}
where we use the fact that $Z \sim \text{Unif}([K])$, Assumption~\ref{as:dgp}
and the linear model for $Y_i$.

Furthermore, by the law of total probability,
\begin{align}
  \label{eq:joint_pdf}
  f_{\btheta}(\bX_{[n]}, Y_{[n]})
  &
    =\frac{f(\bX_{[n]})}{K (2\pi\sigma^2)^{n/2}}
    \sum_{l=1}^K\exp\left\{
    -\frac{1}{2\sigma^2}\sum_{i=1}^n(Y_i-\langle \bX_i, \btheta_k\rangle)^2
    \right\}.
\end{align}
Then, we define the conditional class probability as
\begin{align*}
  w_Z(\btheta) \eqdef
  g_{\btheta}(Z|\bX_{[n]}, Y_{[n]})
  & = \frac{f_{\btheta}(\bX_{[n]}, Y_{[n]}, Z)}{f_{\btheta}(\bX_{[n]}, Y_{[n]})}
  =\frac{
    \exp\{-\frac{1}{2\sigma^2}\sum_{i=1}^n(Y_i-\langle \bX_i, \btheta_Z\rangle)^2\}
    }
    {
    \sum_{l=1}^K\exp\{-\frac{1}{2\sigma^2}\sum_{i=1}^n(Y_i-\langle \bX_i, \btheta_k\rangle)^2\}
    }.
\end{align*}
Recall the definition of $Q$, previously given in~\eqref{eq:pop_Q},
\begin{align*}
  Q(\btheta|\btheta')
  &=
    \int_{\mathcal{X}^n\times\mathcal{Y}^n}
    \left(
    \int_{\mathcal{Z}}g_{\btheta'}(z|\bx_{[n]}, y_{[n]})
    \log f_{\btheta}(\bx_{[n]}, y_{[n]}, z)\diff z
    \right)
    f_{\btheta'}(\bx_{[n]}, y_{[n]})\diff \bx_{[n]}\diff y_{[n]}
  \\
  &=\mathbb{E}_{\bX_{[n]}, Y_{[n]}}
    \left[
    \int_{\mathcal{Z}}g_{\btheta'}(z|\bX_{[n]}, Y_{[n]})
    \log f_{\btheta}(\bX_{[n]}, Y_{[n]}, z)\diff z
    \right]
  \\
  &=\mathbb{E}_{\bX_{[n]}, Y_{[n]}}
    \left[
    \mathbb{E}_{Z \sim g_{\btheta'}(\cdot| \bX_{[n]}, Y_{[n]})}[
    \log f_{\btheta}(\bX_{[n]}, Y_{[n]}, Z)]
    \right].
\end{align*}
Now, plugging in for the density $f_{\btheta}$, as derived
in~\eqref{eq:joint_pdf}, and simplifying,
\begin{align*}
  Q(\btheta|\btheta')
  &
    =
    \mathbb{E}_{\bX_{[n]}, Y_{[n]}}
    \left[
    \sum_{k=1}^Kw_k(\btheta')
    \left( -\frac{1}{2\sigma^2}\sum_{i=1}^n(Y_i-\langle \bX_i,\btheta_k\rangle)^2 \right)
    \right]
\end{align*}
Without loss of generality, we focus on the maximization of $Q$ with respect to the $k$-th vector
$\btheta_k$.
Taking the partial derivative of $Q$ with respect to $\btheta_k$ and seting it
equal to zero:
\begin{align*}
    -\mathbb{E}_{\bX_{[n]}, Y_{[n]}}
    \left[w_k(\btheta')\sum_{i=1}^n\bX_i\bX_i^T\btheta_k\right]
    +
  \mathbb{E}_{\bX_{[n]}, Y_{[n]}}
    \left[w_k(\btheta')\sum_{i=1}^nY_i\bX_i\right],
    = 0
\end{align*}
We can solve for the one-step update for the $k$\textit{-th} vector to be
\begin{align*}
  \btheta^+_k
  =
  \mathbb{E}_{\bX_{[n]}, Y_{[n]}}
    \left[w_k(\btheta')\sum_{i=1}^n\bX_i\bX_i^T\right]^{-1}
  \mathbb{E}_{\bX_{[n]}, Y_{[n]}}
  \left[w_k(\btheta')\sum_{i=1}^nY_i\bX_i\right].
\end{align*}
\end{proof}
\subsection{Proof of Proposition \ref{prop:empirical_EM}}
\begin{proof}
  The proof of this proposition then follows directly by
  taking limits in Proposition~\ref{prop:population_EM}
  Since the only difference between the derivation of the empirical EM iterates and
  the population EM iterates is that the sample averages with respecvt to $m$
  are replaced with respective expectations (see~\eqref{eq:qm} and~\eqref{eq:pop_Q}).
\end{proof}

\section{Proofs for Section~\ref{sec:results}}
\label{appendix:proofs_theorems}
In this section we prove the two theorems presented in Section~\ref{sec:results}.
Throughout this section, when the subscript of the expectation is omitted,
$\mathbb{E}[\cdot]$ denotes the expectation with respect to the joint density of
$(\bX_{[n]}, Y_{[n]})$.
\subsection{Proof of Theorem~\ref{thm:pop_cons}}\label{proof_of_theorem_5}
\begin{proof}
We perform a one-step analysis.
Suppose at the current step, we have estimates
$\{\btheta_k\}_{k=1}^K$, and one iteration of population EM generates new estimates
$\{\btheta^+_k\}_{k=1}^K$. Without loss of generality, we focus on $\btheta_1^+$.
The same steps can be repeated for any of the $K$ vectors.
Pluggin in for $\btheta_1^+$ as defined by Proposition~\ref{prop:population_EM}, we have
\begin{align}
    \label{eq:diff}
  \btheta_1^+-\btheta_1^*
  &=
    \mathbb{E}\left[w_1(\btheta)\sum_{i=1}^n\bX_i\bX_i^T\right]^{-1}
    \mathbb{E}\left[w_1(\btheta)\sum_{i=1}^nY_i\bX_i\right]-\btheta_1^*
    \nonumber
  \\
  &=
    \mathbb{E}\left[w_1(\btheta)\sum_{i=1}^n\bX_i\bX_i^T\right]^{-1}
    \mathbb{E}\left[w_1(\btheta)\sum_{i=1}^n\bX_i(Y_i-\bX_i^T\btheta_1^*)\right].
\end{align}
We now observe that, by definition,
\begin{align*}
    \mathbb{E}\left[w_1(\btheta)\sum_{i=1}^n\bX_i(Y_i-\bX_i^T\btheta_1^*)\right]
  &
    = \mathbb{E}_{1}\left[\sum_{i=1}^n\bX_i\varepsilon_i^1\right]
    = 0.
\end{align*}
Therefore, we can reduce \eqref{eq:diff} to
\begin{align}
  \label{eq:population_iterate}
  \btheta_1^+-\btheta_1^*
  ={\underbrace{
  \mathbb{E}\left[w_1(\btheta)\sum_{i=1}^n\bX_i\bX_i^T\right]}_{A}}^{-1}
  \underbrace{\mathbb{E}\left[(w_1(\btheta)-w_1(\btheta^*))
  \sum_{i=1}^n\bX_i(Y_i-\langle \bX_i, \btheta_1^*\rangle)\right]}_{B} .
\end{align}
Note here that we do not include the inverse in the definition of $A$.
We will now bound each term separately,
starting with the numerator $B$.
\subsection*{Bounding $B$:}
\begin{align*}
  K\|B\|
    =
    K\sup_{s\in\mathcal{S}^{d-1}}
    \left|\mathbb{E}\left[(w_1(\btheta)-w_1(\btheta^*))
    \sum_{i=1}^n(Y_i-\bX_i^T \btheta_1^*)
    \bX_i^Ts\right]\right|
    \le
    |T_1|
    +\sum_{k\neq 1} |T_k|
\end{align*}
where
\begin{align*}
  T_1 &=
        \mathbb{E}_{1}\left[(w_1(\btheta)-w_1(\btheta^*))
        \sum_{i=1}^n(Y_i-\bX_i^T \btheta_1^*)
        \bX_i^Ts\right]
  \\
  T_k&=
       \mathbb{E}_{k}\left[(w_1(\btheta)-w_1(\btheta^*))
    \sum_{i=1}^n(Y_i-\bX_i^T \btheta_1^*)
    \bX_i^Ts\right].
\end{align*}
We start by bounding $T_k$, $\forall k\ne 1$.
\begin{align}
  \label{eq:t_k}
    T_k
  &= |\mathbb{E}_{k}
    [(w_1(\btheta) - w_1(\btheta^*))
    \sum_{i=1}^n(\varepsilon_i+\bX_i^T(\btheta_k^*-\btheta_1^*))\bX_i^T s]|
    \nonumber
  \\
  &\le\underbrace{
    |\mathbb{E}_{k}
    [(w_1(\btheta) - w_1(\btheta^*))
    \sum_{i=1}^n\bX_i^T(\btheta_k^*-\btheta_1^*)\bX_i^T s]|
    }_{T_{k1}}
    +\underbrace{
    |\mathbb{E}_{k}
    [(w_1(\btheta) - w_1(\btheta^*))\sum_{i=1}^n\varepsilon_i\bX_i^T s]|
    }_{T_{k2}}
\end{align}

\subsubsection*{Probability bounds}
In order to bound both terms in the above inequality, we need to define the
following events for any $k\neq 1$:
\begin{align}
  \label{eq:prob_bounds}
  G_{k,1}
  &= \left\{
    \sum_{i=1}^n(\bX_i^T(\btheta_k^*-\btheta_1^*))^2\ge\frac{320\sigma^2n}{3}
    \right\},
    \qquad
     G_3
  = \left\{\sum_{i=1}^n\varepsilon_i^2\le 2\sigma^2n\right\},
  \\
  \nonumber
  G_{k,2}
  &= \left\{
    \max\left\{\sum_{i=1}^n(\bX_i^T(\btheta_k-\btheta_k^*))^2,
    \sum_{i=1}^n(\bX_i^T(\btheta_1-\btheta_1^*))^2\right\}
    \le\frac{1}{16}\sum_{i=1}^n( \bX_i^T(\btheta_k^*-\btheta_1^*))^2\right\}.
\end{align}
We will show that these are
high-probability events that control the magnitude of $T_k$.
We first show that the complements of each of these events have small probabilities.
Starting with $G_{k,1}^c$
\begin{align*}
  \mathbb{P}(G_{k,1}^c)
  = \mathbb{P}
  \left(\sum_{i=1}^n(\bX_i^T(\btheta_k^*-\btheta_1^*))^2\le\frac{320\sigma^2n}{3}\right)
  =
  \mathbb{P}
  \left(\sum_{i=1}^n
  \frac{(\bX_i^T(\btheta_k^*-\btheta_1^*))^2}{\|\btheta_k^*-\btheta_1^*\|^2}
  \le\frac{320\sigma^2n}{3\|\btheta_k^*-\btheta_1^*\|^2}
  \right).
\end{align*}
Note that by Assumption~\ref{as:dgp},
$\frac{(\bX_i^T(\btheta_k^*-\btheta_1^*))^2}{\|\btheta_k^*-\btheta_1^*\|^2} \sim \chi^2_1$.
Then by tail bounds for $\chi^2$ random variables (see
\cite[Corollary of Lemma 1]{10.1214/aos/1015957395}),
with
$s = n\left(\frac{1}{2}-\frac{160\sigma^2}{3\Delta_{\min}^2}\right)^2$,
\begin{align*}
  \mathbb{P}(G_{k,1}^c)
  \leq
  \exp\left(-n\left(\frac{1}{2}-\frac{160\sigma^2}{3\Delta_{\min}^2}\right)^2\right)
  \leq
  \exp(-\frac{n}{K^2}),
\end{align*}
by the assumption placed on the signal-to-noise ratio.
Now, for $G_{k, 2}^c$,
\begin{align*}
  &
  \mathbb{P}(G_{k,2}^c)
    \leq
    \mathbb{P}\left(
    \sum_{i=1}^n(\bX_i^T(\btheta_k-\btheta_k^*))^2
    \geq
    \frac{1}{16}\sum_{i=1}^n(\bX_i^T(\btheta_k^*-\btheta_1^*))^2
    \right)
  \\
  &
    \qquad \qquad \qquad \qquad \qquad \qquad
    +\mathbb{P}
    \left(\sum_{i=1}^n( \bX_i^T(\btheta_1-\btheta_1^*))^2
    \geq
    \frac{1}{16}\sum_{i=1}^n(\bX_i^T(\btheta_k^*-\btheta_1^*))^2
    \right).
\end{align*}
Note that $\forall t>0$, the first term is bounded as
{\small
\begin{align}
  \label{eq:g_k2c}
  &
  \mathbb{P}
  \left(
    \sum_{i=1}^n(\bX_i^T(\btheta_k-\btheta_k^*))^2
    \geq
    \frac{1}{16}\sum_{i=1}^n(\bX_i^T(\btheta_k^*-\btheta_1^*))^2
    \right)
  \\
  &
    \leq
    \mathbb{P}
    \left(
    \sum_{i=1}^n
    \frac{\sum_{i=1}^n(\bX_i^T(\btheta_k-\btheta_k^*))^2}
    {\|\btheta_k-\btheta_k^*\|^2}
    \geq
    \frac{t}{\|\btheta_k-\btheta_k^*\|^2}
    \right)
    +\mathbb{P}
    \left(
    \frac{1}{16}\sum_{i=1}^n
    \frac{(\bX_i^T(\btheta_k^*-\btheta_1^*))^2}
    {\|\btheta_k^*-\btheta_1^*\|^2}
    \leq
    \frac{t}{\|\btheta_k^*-\btheta_1^*\|^2}
    \right).
    \nonumber
\end{align}
}
Once again the $\chi^2$ tail bounds from \cite[Corollary of Lemma 1]{10.1214/aos/1015957395} can be applied by
choise of
\begin{align*}
  s
  =
  \frac{t}{2\|\btheta_k-\btheta_k^*\|^2}
  -
  \frac{\sqrt{n} }{2}
  \sqrt{\frac{2t}{\|\btheta_k-\btheta_k^*\|^2} -n}
\end{align*}
for the first probability bound and
\begin{align*}
  s = \frac{n}{4} - \frac{8t}{\|\btheta_k^* - \btheta_1^*\|^2} +\frac{64t^2}{n\|\btheta_k^* - \btheta_1^*\|^4}
\end{align*}
for the second probability bound.
In order for the bounds to be non-trivial, we need
$t>n\|\btheta_k-\btheta_k^*\|^2$
and
$t<\frac{1}{16}n\|\btheta_k^*-\btheta_1^*\|^2$ to hold simultaneously.
By Assumption~\ref{as:identifiability}, both conditions on $t$ can be satisfied
by simply choosing
$t = \frac{1}{2}n(\|\btheta_k-\btheta_k^*\|^2+\frac{1}{16}\|\btheta_k^*-\btheta_1^*\|^2)$.
Thus,
\begin{align*}
  \eqref{eq:g_k2c}
  &
    \leq
    \exp(-\frac{n(1-4\alpha)^2}{64 \alpha^2})
    +
    \exp(-\frac{n(1-16 \alpha^2)^2}{16})
    \leq 2\exp(-C_{\alpha}n),
\end{align*}
where
\begin{align}
  \label{eq:c_alpha}
C_\alpha = \frac{(1-4\alpha)^2}{64 \alpha^2}.
\end{align}
Finally, for $G_3^c$, we again employ \cite[Corollary of Lemma 1]{10.1214/aos/1015957395} to obtain
\begin{align*}
  \mathbb{P}(G_3^c)
  &=
    \mathbb{P}
    \left(
    \sum_{i=1}^n
    \varepsilon_i^2 \geq 2n\sigma^2
    \right)
    \leq
    \exp(-n).
\end{align*}
Now, let $G_k = G_{k,1}\cap G_{k,2}\cap G_3$ be the intersection of the three events.
And thus,
\begin{align*}
  \Prob(G_k)
  = 1 - \Prob(G_{k,1}^c) - \Prob(G_{k,2}^c) - \Prob(G_{k,3}^c)
  \geq 1-2\exp(-n) - \exp(-\frac{n}{K^2}).
\end{align*}
We will use this to partition our computation of expectations into different
regions and bound each term separately next.
\subsubsection*{Expectations}
Recall $T_{k,1}$ and $T_{k,2}$ as defined in~\eqref{eq:t_k}.
We partition each of these terms by the $\{G_{k, l}\}_{l=1}^3$ sets as defined earlier.
To avoid repetition, we will only show the bounding argument for $T_{k,1}$,
the same methodology applies for $T_{k,2}$ and results in a bound of the same order.
\begin{align}
  T_{k,1}
  &\le
    \mathbb{E}_{k}
    \left[
    |(w_1(\btheta) - w_1(\btheta^*))
    \sum_{i=1}^n\bX_i^T(\btheta_k^*-\btheta_1^*)\bX_i^T s
    \mid \mathds{1}_{G_k}\right]
    \label{eq:tk_gk}
  \\
  &
    +\mathbb{E}_{k}\left[|(w_1(\btheta) - w_1(\btheta^*))
    \sum_{i=1}^n\bX_i^T(\btheta_k^*-\btheta_1^*)\bX_i^T s\mid
    \mathds{1}_{G_{k,1}^c}\right]
    \label{eq:tk_gk1c}
  \\
  &
    +\mathbb{E}_{k}\left[|(w_1(\btheta) - w_1(\btheta^*))
    \sum_{i=1}^n\bX_i^T(\btheta_k^*-\btheta_1^*)\bX_i^T s\mid
    \mathds{1}_{G_{k,2}^c}\right]
    \label{eq:tk_gk2c}
  \\
  &
    +\mathbb{E}_{k}\left[|(w_1(\btheta) - w_1(\btheta^*))
    \sum_{i=1}^n\bX_i^T(\btheta_k^*-\btheta_1^*)\bX_i^T s\mid
    \mathds{1}_{G_{3}^c}\right].
    \label{eq:tk_gk3c}
\end{align}
Starting with the weights in~\eqref{eq:tk_gk}
\begin{align}
  \label{bound_on_w_1_on_G_k}
  w_1(\btheta)
  &
    \le \exp\left(
    \frac{1}{2\sigma^2}\sum_{i=1}^n(Y_i-\bX_i^T\btheta_k)^2
    -\frac{1}{2\sigma^2}\sum_{i=1}^n(Y_i-\bX_i^T\btheta_1)^2
    \right)\nonumber
  \\
  &=
    {\small
    \exp\left(
    \frac{1}{2\sigma^2}
    \sum_{i=1}^n(\varepsilon_i+\bX_i^T(\btheta_k^*-\btheta_k))^2
    - \frac{1}{2\sigma^2}
    \sum_{i=1}^n(\varepsilon_i
    +\bX_i^T(\btheta_k^*-\btheta_1^*)
    -\bX_i^T(\btheta_1-\btheta_1^*))^2
    \right)\nonumber
}
  \\
  &
    \leq
    \exp\left(
    \frac{3}{2\sigma^2}\sum_{i=1}^n\varepsilon_i^2
     - \frac{3}{64\sigma^2}\sum_{i=1}^n(\bX_i^T( \btheta_k^*-\btheta_1^*))^2\right)
\end{align}
The last inequality in~\eqref{bound_on_w_1_on_G_k} follows from applying $(a+b)^2 \leq 2a^2 + 2b^2$ and
observing that
\begin{align*}
    \sum_{i=1}^n(\varepsilon_i
    +\bX_i^T(\btheta_k^*-\btheta_1^*)
    -\bX_i^T(\btheta_1-\btheta_1^*))^2
  &
    \geq\sum_{i=1}^n\frac{1}{2}(
    \bX_i^T(\btheta_k^*-\btheta_1^*) -
    \bX_i^T(\btheta_1-\btheta_1^*))^2
    -\varepsilon_i^2
  \\
  & \ge\frac{7}{32}\sum_{i=1}^n
    (\bX_i^T(\btheta_k^*-\btheta_1^*))^2
    -\sum_{i=1}^n\varepsilon_i^2.
\end{align*}
Then, by definition of $G_k$, we see that \eqref{bound_on_w_1_on_G_k} is
ultimately bounded from above by $\exp(-2n)$.
The same exercise can be repeated for $w_1(\btheta^*)$ to get an identical bound,
which is crude, but sufficient for our purposes.
Therefore, $|w_1(\btheta)|+|w_1(\btheta^*)|\leq \exp(-n)$.
Using this, we can see that
\begin{align*}
  \eqref{eq:tk_gk}
  &
  \leq
  e^{-n}
  \mathbb{E}\left[
  \sum_{i=1}^n(\bX_i^T(\btheta_k^*-\btheta_1^*)) ^2
  \sum_{i=1}^n(\bX_i^T s)^2
  \right]^{1/2}
  \\
  &
  \leq
    e^{-n}
    \left( n\|\btheta_k^*-\btheta_1^*\|^2
    +n(n-1)\|\btheta_k^*-\btheta_1^*\|^2 \right)^{1/2}
    =O(\Delta_{\max}n e^{-n}),
\end{align*}
where the first inequality follows by the bounds on $w_1(\btheta),
w_1(\btheta^*)$ and the Cauchy-Schwarz inequality and the second
inequality follows by \cite[Lemma 7]{supp_balakrishnan}.

Now, we turn to the remaining terms of $T_{k1}$ (\eqref{eq:tk_gk1c}
-\eqref{eq:tk_gk3c}).
\begin{align*}
  \eqref{eq:tk_gk1c}
  &
  \leq \sqrt{\mathbb{E}_{k}
    \left[\sum_{i=1}^n (\bX_i^T(\btheta_k^*-\btheta_1^*))^2|G_{k,1}^c \right]}
    \sqrt{\mathbb{E}_{k}
    \left[\sum_{i=1}^n(\bX_i^T s)^2|G_{k, 1}^c
    \right]}
  \mathbb{P}(G_{k, 1}^c)
  \\
  &
  \leq O(\Delta_{\max}\sqrt{n}\exp(-n/K^2)),
\end{align*}
where the last inequality follows from Lemma~\ref{lemma.1}.
Next,
\begin{align*}
  \eqref{eq:tk_gk2c}
    &\le\sqrt{\mathbb{E}_{k}
      \left[\sum_{i=1}^n(\bX_i^T(\btheta_k-\btheta_k^*))^2|G_{k,2}^c\right]
      +\mathbb{E}_{k}
      \left[\sum_{i=1}^n(\bX_i^T(\btheta_1-\btheta_1^*))^2|G_{k,2}^c\right]}
      \mathbb{P}(G_{k, 2}^c)
  \\
    &\leq
      O(\alpha n \Delta_{\min}\exp(-C_{\alpha}n)),
\end{align*}
where the last line follows from Lemma~\ref{lemma.2} and $C_{\alpha}$is defined in~\eqref{eq:c_alpha}.
Finally,
\begin{align*}
  \eqref{eq:tk_gk3c}
  \leq
    \sqrt{
    \mathbb{E}_{k}
    \left[\sum_{i=1}^n(\bX_i^T( \btheta_k^*-\btheta_1^*))^2\right]}
    \sqrt{\mathbb{E}_{k}
    \left[\sum_{i=1}^n(\bX_i^T s)^2\right]}
    \mathbb{P}(G_3^c)
    \leq O(\Delta_{\max}n\exp(-n))
\end{align*}
follows from $\{\bX_i\}_{i=1}^n$ being independent of the event $G_3^c$.
Therefore, putting all terms together,
\begin{align*}
  T_{k, 1}
  \leq
  O
  \left(
  \Delta_{\max}n\exp(-n)
  +
  \alpha n \Delta_{\min}\exp(-C_{\alpha}n)
  +\Delta_{\max}\sqrt{n}\exp(-n/K^2)
  \right).
\end{align*}
Similarly analysis yields the following bound for $T_{k, 2}$:
\begin{align*}
  T_{k, 2}
  \leq
  O\left(
  (\sigma+\Delta_{\max})n\exp(-n)
  +
  \alpha n \Delta_{\min}\exp(-C_{\alpha}n)
  \right).
\end{align*}

The final term for bounding $B$ is $T_1$, which can be treated similar to $T_k$.
First, applying Cauchy-Schwarz,
\begin{align}
  \label{eq:t1_cs}
    T_1
    \leq
    \mathbb{E}_{1}
    [(w_1(\btheta) - w_1(\btheta^*))^2]^{1/2}
    \mathbb{E}_{1}
    [(\sum_{i=1}^n\varepsilon_i \bX_i^Ts)^2]^{1/2}.
\end{align}
It is straightforward to see that the second expectation in~\eqref{eq:t1_cs} is
equal to $n \sigma^2$.
Now, for evaluating the first expectation, we repeat the partitioning and
conditioning exercise done for the $T_{k,1}$ term.
We will use $G_1, G_2, G_3$ to
denote the three event sets.
\begin{align*}
  G_1
  &= \big\{
    \sum_{i=1}^n(\bX_i,^T (\btheta_k^*-\btheta_1^*))^2
    \geq
    \frac{320\sigma^2n}{3},\ \forall k\ne 1
    \big\},
    \\
  G_2
  &= \{
    \sum_{i=1}^n(\bX_i^T(\btheta_1-\btheta_1^*)) ^2
    \leq
    \frac{1}{16}\sum_{i=1}^n(\bX_i^T(\btheta_k^*-\btheta_1^*))^2, \ \forall k\ne 1
    \},
\end{align*}
and $G = G_1\cap G_2\cap G_3$ ($G_3$ was defined earlier in the probability
bounds for $T_k$).
Now, using the observation that
$G_1 = \cap_{k\ne 1}G_{k, 1}$, and \\
$G_2 =
\cap_{k\ne 1}
\left\{\sum_{i=1}^n(\bX_i^T(\btheta_1-\btheta_1^*))^2
  \leq
  \frac{1}{16}\sum_{i=1}^n(\bX_i^T(\btheta_k^*-\btheta_1^*))^2
\right\}$.
We can directly use the previous calculations for
$G_{k, 1}^c$ and $G_{k, 2}^c$
to show exponential concentration of $G_1$ and $G_2$.
That is,
$ \mathbb{P}(G_1^c) \leq \sum_{k\ne 1}\mathbb{P}(G_{k, 1}^c) \leq K\exp(-n/K^2) $
and
\begin{align*}
  \mathbb{P}(G_2^c)
  &
    \leq
    \sum_{k\ne 1}\mathbb{P}
    \left(\sum_{i=1}^n(\bX_i^T(\btheta_1-\btheta_1^*))^2
    \geq
    \frac{1}{16}\sum_{i=1}^n(\bX_i^T(\btheta_k^*-\btheta_1^*))^2
    \right)
    \leq
      2K\exp(-C_{\alpha}n).
\end{align*}
Therefore,
$ \mathbb{P}(G^c) \leq K\exp(-n/K^2)+\exp(-n)+2K\exp(-C_{\alpha}n).$
Next, note that
\begin{align*}
  \mathbb{E}_{1}
  \left[(w_1(\btheta) - w_1(\btheta^*))^2\right]^{1/2}
  \leq
  \mathbb{E}_{1}[(w_1(\btheta) - w_1(\btheta^*))^2|G]+\mathbb{P}(G^c).
\end{align*}
Observe that
$ w_1(\btheta) = 1-\sum_{k\ne 1}w_k(\btheta) \geq 1 - (K-1) \exp(-n)$
on the event $G$.
Similarly, $w_1(\btheta^*)\geq 1-(K-1)\exp(-n)$.
This directly gives the bound
$\mathbb{E}[(w_1(\btheta) - w_1(\btheta^*))^2]\lesssim K^2\exp(-n)$ on the event $G$.
Thus,
\begin{align*}
  T_1 = O(\sqrt{n}\sigma K (\exp(-n) + \exp(-n/K^2) + \exp(-C_\alpha n))).
\end{align*}
Putting all the terms together, we can bound $B$
\begin{align*}
  \|B\|
   & \lesssim
     \frac{n(\sigma+\Delta_{\max})}{K}e^{-n}
    +\frac{\alpha n^{3/2} \Delta_{\min}\sigma}{K}e^{-C_\alpha n}
    +\frac{\sqrt{n}\Delta_{\max}}{K}e^{-n/K^2}
     +K\sigma\sqrt{n}e^{-n}.
\end{align*}
\subsection*{Bound on $A$}\label{population.2}
Recall we define $A$ as
\begin{align*}
    A
  &=
    \frac{1}{K}\sum_{k=1}^K
    \mathbb{E}_{k}
    [w_1(\btheta)\sum_{i=1}^n\bX_i\bX_i^T]
    \geq
    \frac{1}{K}
    \mathbb{E}_{1}
    [
    w_1(\btheta)\sum_{i=1}^n\bX_i\bX_i^T
    ].
\end{align*}
Then,
\begin{align*}
  \|\mathbb{E}_{1} [w_1(\btheta)\sum_{i=1}^n\bX_i\bX_i^T]\|
  \geq
  \|\mathbb{E}_{1} [(1-(K-1)\exp(-n))\sum_{i=1}^n\bX_i\bX_i^T]\|
  =
  n(1-(K-1)e^{-n})
\end{align*}
Thus,
\begin{align}
  \label{eq:A_bound}
  \|A\|^{-1}
  &
    \leq
    \frac{K}{n(1-(K-1)e^{-n})}.
\end{align}
We can bring the bounds on $A$ and $B$ together to see that
\begin{align*}
  \|\btheta_1^+-\btheta_1^*\|
  &
  \le\|A\|^{-1}\|B\|
  \\
  &\lesssim
    (\sigma+\Delta_{\max}+\frac{K^2\sigma}{\sqrt{n}})\frac{e^{-n}}{1-(K-1)e^{-n}}
    +(\alpha \Delta_{\min}\sqrt{n} + \frac{K^2}{\sqrt{n}})\frac{\sigma e^{-C_\alpha n}}{1-(K-1)e^{-n}}
  \\
  &\qquad \qquad\qquad
    +(\Delta_{\max} + K^2\sigma)\frac{e^{-n/K^2}}{\sqrt{n}(1-(K-1)e^{-n})}.
\end{align*}
This rate can be simplified based on any additional assumptions one is willing to make on
$\alpha, K, \sigma, \Delta_{\max}$ and $\Delta_{\min}$.
In particular, we note that this bound allows for $K$ to increase with $n$ at a
rate of $o(n)$. This bound also allows for $\Delta_{\min}$ and $\Delta_{\max}$ to
evolve with $n$ up to exponential order.
\end{proof}

\subsection{Proof of Theorem~\ref{thm:emp_consistency}}
\label{proof:empirical_thm}
\begin{proof}
Similar to the proof of Theorem~\ref{thm:pop_cons}, we perform a one-step analysis,
focusing on $k=1$, without loss of generality.
To simplify some of the notation we will use $\E_n$ and $\E_m$to denote the empirical
expectation with over $n$ and $m$, respectively.
\begin{align}
  \label{eq:empirical_decomp}
  \widehat\btheta_1^+-\btheta_1^*
  &=
    {\underbrace{
    \E_m[w_1(\widehat\btheta)\E_n[\bX_i^j\bX_i^{jT}]]}_{\hat{A}}}^{-1}
    \underbrace{\E_m[w_1(\widehat\btheta)\E_n[\bX_i^j(Y_i^j- \bX_i^{jT}\btheta_1^*)]]}_{\hat{B}}.
\end{align}
Note here that we do not include the inverse in the definition of $\hat{A}$.
\subsection*{Bounding $\hat{B}$}
To leverage the results of Theorem~\ref{thm:pop_cons},
we add and subtract $B$~\eqref{eq:population_iterate} to $\hat{B}$,
i.e.\ $\hat{B} = (\hat{B}-B)+B$.
Since $B$ was bounded in Theorem~\ref{thm:pop_cons}, we only need to study here
$\hat{B} - B$.
The final bound will be obtained by combining the two parts as
$\|\hat{B}\|\le \|\hat{B}-B\|+\|B\|$.
Note that we can start by defininge $\hat{B}$ as
\begin{align}
  \label{eq:Bhat_partition}
  &\E_m[w_1(\widehat\btheta)\E_n[\bX_i^j(Y_i^j- \bX_i^{jT}\btheta_1^*)]]
    \nonumber
  \\
  &
  = \sum_{k=1}^K\E_{m,k}[w_1(\widehat\btheta)\E_n[\bX_i^j(Y_i^j- \bX_i^{jT}\btheta_1^*)]]
    \nonumber
  \\
  &=
    \E_{m,1}[w_1^j(\widehat\btheta) \E_{n}[\bX_i^j\varepsilon_i^j]]
    +
    \sum_{k=2}^K\E_{m,k}[w_1^j(\widehat\btheta) \E_n[\bX_i^j(\bX_i^{jT}(\btheta_k^* - \btheta_1^*)+ \varepsilon_i^j)]],
\end{align}
where $E_{m,k}$ corresponds to the $m$-\textit{th} client having data generated
from the $k$-\textit{th} mixture element.
Note that to study the each term in~\eqref{eq:Bhat_partition},
we would like to apply Lemma~\ref{lemma.8} to bound the deviation of the empirical mean
from the population mean (the corresponding term in $B$) with high probability ($1-\delta$).
But, in order to do so, we need to first show that each term is
sub-exponential with a finite sub-exponential norm.
Recall that for a random variable $W$, its sub-exponential norm is defined as
\begin{align*}
  \|W\|_{\psi_1} = \inf \{ k > 0: \E[\exp(|W|/k)] \leq 2\}.
\end{align*}
Starting with the first term, we compute that it's sub-exponential norm is given by
\begin{align}
  \label{eq:subexp_t1}
  \|w_1^j(\widehat\btheta) \E_{n}[\bX_i^j\varepsilon_i^j]\|_{\psi_1} = O(\frac{\sigma}{\sqrt{n}}),
\end{align}
The probability of the data being generated by the first mixture element is
$1/K$ by definition.
Thus, Lemma~\ref{lemma.8} applies with the parameters $p=1/K$ and
sub-exponential norm given by~\eqref{eq:subexp_t1} for
\begin{align*}
  t=O\left(\frac{\sigma}{\sqrt{n}}\sqrt{\frac{d\log(dK^2/\delta)}{m}}
  \sqrt{\frac{1}{K} \vee \frac{\log(dK^2/\delta)}{m}}
  \right).
\end{align*}

In order to simplify the computation for the second term
of~\eqref{eq:Bhat_partition}, we partition the inner
sample expectation (w.r.t.\ $n$) based on the events $G_k$,
$\{G_{k,l}\}_{l=1}^3$.
\\
That is,
{\footnotesize
\begin{align*}
  \E_{n,k}[\bX_i^j(\bX_i^{jT}(\btheta_k^* - \btheta_1^*)+ \varepsilon_i^j)]
  &=
    \underbrace{\E_{n,k}[\bX_i^j(\bX_i^{jT}(\btheta_k^* - \btheta_1^*)+ \varepsilon_i^j)|G_k]}_{(I)}
    + \underbrace{\E_{n,k}[\bX_i^j(\bX_i^{jT}(\btheta_k^* - \btheta_1^*)+ \varepsilon_i^j)|G_{k,1}^c]}_{(II)}
  \\
  &
    + \underbrace{\E_{n,k}[\bX_i^j(\bX_i^{jT}(\btheta_k^* - \btheta_1^*)+ \varepsilon_i^j)|G_{k,2}^c]}_{(III)}
    + \underbrace{\E_{n,k}[\bX_i^j(\bX_i^{jT}(\btheta_k^* - \btheta_1^*)+ \varepsilon_i^j)|G_{k,3}^c]}_{(IV)}
\end{align*}
}
where
\begin{align}
  \label{eq:prob_bounds_empirical}
  G_{k,1}
  &= \left\{
    \sum_{i=1}^n(\bX_i^T(\btheta_k^*-\btheta_1^*))^2\ge\frac{320\sigma^2n}{3}
    \right\},
    \qquad
    G_3
    = \left\{\sum_{i=1}^n\varepsilon_i^2\le 2\sigma^2n\right\}.
  \\
  \nonumber
  G_{k,2}
  &= \left\{
    \max\left\{\sum_{i=1}^n(\bX_i^T(\widehat\btheta_k-\btheta_k^*))^2,
    \sum_{i=1}^n(\bX_i^T(\widehat\btheta_1-\btheta_1^*))^2\right\}
    \le\frac{1}{16}\sum_{i=1}^n( \bX_i^T(\btheta_k^*-\btheta_1^*))^2\right\},
\end{align}
and $G_k = G_{k, 1}\cap G_{k,2}\cap G_3$.
Note that these events are identical to the events defined
in~\eqref{eq:prob_bounds} with the population iterate replaced by the empirical iterate.
We now show finite sub-exponential norms for each of the 4 terms above.
\subsubsection*{Analysis of $(I)$:}
\begin{align}
  \label{eq:I}
  w_1^j(\widehat\btheta)(I)
  = w_1^j(\widehat\btheta)\E_{n, k}[\bX_i^j\bX_i^{jT}(\btheta_k^* - \btheta_1^*)|G_k]
  + w_1^j(\widehat\btheta)\E_{n, k}[\bX_i^{jT}\varepsilon_i^j|G_k]
\end{align}
Note that $\Prob(Z=k|G_k) \leq \Prob(Z=k) = \frac{1}{K}$.
Thus, Lemma~\ref{lemma.8} holds for the second term of~\eqref{eq:I} with $p = 1/K$,
sub-exponential norm $O(\frac{\sigma}{\sqrt{n}}\exp(-n)),$
and
\begin{align*}
  t=O\left(\sigma\exp(-n)\sqrt{\frac{d\log(dK^2/\delta)}{nm}}\sqrt{\frac{1}{K} \vee \frac{\log(dK^2/\delta)}{m}}\right).
\end{align*}
For the first term of~\eqref{eq:I},
note that bounding the sub-exponential norm is equivalent to bounding the
sub-exponential norm of the inner product of the element with $s \in S^{d-1}$.
That is,
\begin{align}
  &
    \|w_1^j(\widehat\btheta)
    \E_{n, k}[\bX_i^j\bX_i^{jT}(\btheta_k^* - \btheta_1^*)|G_k] \|_{\psi_1}
    \stackrel{(\rm i)}{\leq}
    e^{-2n}
    \sup_{q\ge 1}\frac{1}{q}
    \E[|\E_{n,k}[(\bX_i^{jT}s)\bX_i^{jT}(\btheta_k^* - \btheta_1^*)|^q|G_k]]^{\frac{1}{q}}
    \nonumber
  \\
  &
    \stackrel{(\rm ii)}{\leq}
    e^{-2n}
    \sup_{q\ge 1}\frac{1}{q}
    \sqrt{\E[\E_{n,k}[(\bX_i^{jT}s)^2|G_k]^{q}]^{1/q}}
    \sqrt{\mathbb{E}[\E_{n,k}[(\bX_i^{jT}(\btheta_k^* - \btheta_1^*))^2|G_k]^{q}]^{1/q}}
    \nonumber
  \\
  &\stackrel{(\rm iii)}{\leq}
    O(\Delta_{\max}n^{-1/2} e^{-2n})
  \label{eq:subEbound}
\end{align}
Inequality $(\rm i)$ follows from the fact that on event $G_k$,
$w_1^j(\widehat\btheta)\leq\exp(-2n)$ (see \eqref{bound_on_w_1_on_G_k}).
$(\rm ii)$
follows from applying Cauchy-Schwarz twice. $(\rm iii)$ follows from the fact that all
$\bX_i^j$ are independent of the event $\{Z = k\}$
and $\E_{n,k}[(\bX_i^{jT} s)^2]$ is independent of $G_k$,
$\mathbb{P}(G_k)> \frac{1}{2}$ for $n$ large enough (see analysis of $G_k$ in
the proof of Theorem~\ref{thm:pop_cons}) and the fact that
$\E_{n,k}[(\bX_i^{jT} s)^2]\sim\text{SubE}(4n, 4)$ and
$\E_{n,k}[(\bX_i^{jT}(\btheta_k^*-\btheta_1^*))^2|G_k]\sim
\text{SubE}(4n\|\btheta_k^*-\btheta_1^*\|^4, 4\|\btheta_k^*-\btheta_1^*\|^2)$.
Thus Lemma~\ref{lemma.8} applies to~\eqref{eq:I}
for
\begin{align*}
  t &=O\Bigg(\frac{\Delta_{\max}e^{-2n}+\sigma e^{-n}}{\sqrt{n}}
      \sqrt{\frac{d\log(dK^2/\delta)}{m}}
      \sqrt{\frac{1}{K} \vee \frac{\log(dK^2/\delta)}{m}}
      \Bigg).
\end{align*}

\subsubsection*{Analysis of $(II)$}
\begin{align}
  \label{eq:II}
  w_1^j(\widehat\btheta)(II)
  = w_1^j(\widehat\btheta)\E_{n,k}[\bX_i^j(\bX_i^{jT}(\btheta_k^* - \btheta_1^*)|G_{k,1}^c]
  + w_1^j(\widehat\btheta)\E_{n, k}[\bX_i^{jT}\varepsilon_i^j|G_{k,1}^c]
\end{align}
Define $p\leq \mathbb{P}(G_{k,1}^c)\leq\exp(-\frac{n}{16})$.
We note that the assumption of $p \leq 1/K$ is satisfied for Lemma \ref{lemma.8},
by the assumption that $n \gtrsim \log(K)$.
For the second term of~\eqref{eq:II}, the sub-exponential norm is
of the order
$\frac{\sigma}{\sqrt{n}}$ and so Lemma~\ref{lemma.8} holds with
\begin{align*}
  t &= O\Big(\frac{\sigma}{\sqrt{n}}
      \sqrt{e^{-n/16}\vee \frac{\log(dK^2/\delta)}{m}}
      \sqrt{\frac{d\log(dK^2/\delta)}{m}}
      \Big).
\end{align*}
Then, for the first term of~\eqref{eq:II}, the sub-exponential norm is bounded
by repeated application of Cauchy-Schwarz:
\begin{align*}
    \|w_1^j(\widehat\btheta)\E_{n,k}[(\bX_i^{jT}s)(\bX_i^{jT}(\btheta_k^* - \btheta_1^*)|G_{k,1}^c]\|_{\psi_1}
  &\leq
    \sqrt{\frac{320\sigma^2}{3n}}
    \sup_{q\ge 1}\frac{1}{q}
    \mathbb{E}_{k}[\E_n[(\bX_i^{jT}s)^2]^{\frac{q}{2}}|G_{k,1}^c]^{\frac{1}{q}}
  \\
  &\leq
    O(\sigma n^{-1/4})
\end{align*}
Then, Lemma~\ref{lemma.8} holds for~\eqref{eq:II} with
\begin{align*}
  t
  =
  O\Big( \sigma (n^{-1/4} +n^{-1/2})
  \sqrt{e^{-n/16} \vee \frac{\log(dK^2/\delta)}{m}}
  \sqrt{\frac{\log(dK^2/\delta)}{m}}
  \Big).
\end{align*}
\subsubsection*{Analysis of $(III)$:}
\begin{align}
  \label{eq:III}
  w_1^j(\widehat\btheta)(III)
  = w_1^j(\widehat\btheta)\E_{n,k}[\bX_i^j\bX_i^{jT}(\btheta_k^* - \btheta_1^*)|G_{k,2}^c]
  + w_1^j(\widehat\btheta)\E_{n, k}[\bX_i^{jT}\varepsilon_i^j|G_{k,2}^c]
\end{align}
Note $ p=\mathbb{P}_k(G_{k,2}^c)\leq 2\exp(-C_{\alpha}n)$ ($C_\alpha$ defined in~\eqref{eq:c_alpha}).
By standard calculations, the second term of~\eqref{eq:III} has sub-exponential
norm of order $n^{-1/4}\sigma$.
For the first term of~\eqref{eq:III} by
\cite[Lemma 2.7.7]{vershynin2018high},
we have $\forall s\in \mathcal{S}^{d-1}$,
\begin{align}
  &
    \|w_1^j(\widehat\btheta)\E_{n,k}[(\bX_i^{jT}s)\bX_i^{jT}(\btheta_k^* -
    \btheta_1^*)|G_{k,2}^c]\|_{\psi_1}
    \nonumber
  \\
  &
    \leq
    \|\E_n[(\bX_i^{jT} s)^2|G_{k,2}^c]^{1/2} \|_{\psi_2}
    \| \E_n[(\bX_i^{jT}(\btheta_k^*-\btheta_1^*))^2|G_{k,2}^c]^{1/2}\|_{\psi_2}.
    \label{eq:gk_c bound}
\end{align}
By definition of $G_{k, 2}^c$ and \cite[Lemma 2.7.6]{vershynin2018high}, the
second term on the RHS of
\eqref{eq:gk_c bound} can be bounded by
\begin{align*}
  &
    \| \E_n[(\bX_i^{jT}(\btheta_k^*-\btheta_1^*))^2|G_{k,2}^c]^{1/2}\|_{\psi_2}
    \\
  &
    \leq
    (\|16\E_n[(\bX_i^{jT}(\widehat\btheta_k-\btheta_k^*))^2|G_{k,2}^c]\|_{\psi_1}
    + 16\E_n[(\bX_i^{jT}(\widehat\btheta_1-\btheta_1^*))^2|G_{k,2}^c]\|_{\psi_1})^{1/2}
  \\
  &
    =O(n^{-1/4} D_M).
\end{align*}
Moreover, by applying \cite[Lemma 2.7.6]{vershynin2018high} to the first term of~\eqref{eq:gk_c bound},
\begin{align*}
  \|\E_n[(\bX_i^{jT} s)^2|G_{k,2}^c]^{\frac{1}{2}} \|_{\psi_2}
   & =\|\E_n[(\bX_i^{jT} s)^2|G_{k,2}^c]\|^{\frac{1}{2}}_{\psi_1}
  \\
  & =(\sup_{q\ge 1}\frac{1}{q}\E[\E_n[(\bX_i^{jT}
s)^2|G_{k,2}^c]^q|G_{k,2}^c]^{\frac{1}{q}})^{\frac{1}{2}}
    =O(1),
\end{align*}
where the last equality follows from Lemma~\ref{lemma.2}. Therefore, the
sub-exponential norm of the first term in~\eqref{eq:III} is of order $O(n^{-1/4}D_M)$.
Then, Lemma \ref{lemma.8} holds for~\eqref{eq:III} with
\begin{align*}
  t
  =
  O(n^{-1/4}(D_M +\sigma)\sqrt{\exp(-C_{\alpha}n) \vee \frac{\log(dK^2/\delta)}{m}}
  \sqrt{\frac{d\log(dK^2/\delta)}{m}})
\end{align*}
\subsubsection*{Analysis of $(IV)$:}
\begin{align}
  \label{eq:IV}
  w_1^j(\widehat{\btheta})(IV)
  =
  w_1^j(\widehat{\btheta})\E_{n,k}[\bX_i^j\bX_i^{jT}(\btheta_k^* -
  \btheta_1^*)|G_{k,3}^c]
  +
  w_1^j(\widehat{\btheta})\E_{n,k}[\bX_i^j\varepsilon_i^j|G_{k,3}^c]
\end{align}
Note that the sub-exponential norm of the first term is of order
$n^{-1/2} \Delta_{\max}$.
For the second term in~\eqref{eq:IV}, the sub-exponential norm computation is
more involved than in previous cases.
We start by applying Cauchy-Schwarz
$\forall s\in\mathcal{S}^{d-1}$,
{\small
\begin{align}
  \label{eq:IV.1}
  \|w_1^j(\widehat{\btheta})\E_{n,k}[s\bX_i^{j}\varepsilon_i^j|G_{k,3}^c]\|_{\psi_1}
      \leq
  \frac{1}{n}
  \|\sum_i(\bX_i^{jT}s)^2\mathbf{1}(Z_i = k,G_{k,3}^c)\|_{\psi_1}^{1/2}
  \|\sum_i\varepsilon_i^{j2}\mathbf{1}(Z_i = k,G_{k,3}^c)\|^{1/2}_{\psi_1}
\end{align}
}
The first term in~\eqref{eq:IV.1}, we have already seen is of order $n^{1/4}$.
The second term, we notice can be written as
\begin{align}
  \label{eq:epsbound}
  \|\sum_i\varepsilon_i^{j2}\mathbf{1}(Z_i = k,G_{k,3}^c)\|^{1/2}_{\psi_1}
  = \sup_{q\ge 1}\frac{1}{q}
  \sum_i\varepsilon_i^{j2q} \mathbf{1}(Z_i = k,G_{k,3}^c)^{1/q} \mathbb{P}(G_3^c)^{-1/q}.
\end{align}
Now, we decompose $G_3^c$ into
$G_{31}^c = \{12\sigma^2n\ge \sum_{i=1}^n\varepsilon_i^{j2}\ge 2\sigma^2n\}$ and
$G_{32}^c = \{\sum_{i=1}^n\varepsilon_i^{j2}\ge 12\sigma^2n\}$.
Then,
\begin{align*}
  \sum_i\varepsilon_i^{j2q} \mathds{1}(Z_i = k, G_{k,3}^c)^{1/q}
  &
    \le
    \sum_i\varepsilon_i^{j2q} \mathds{1}(Z_i = k, G_{31}^c)]^{1/q}
    +
    \sum_i\varepsilon_i^{j2q} \mathds{1}(Z_i = k, G_{32}^c)]^{1/q}
    \\
    &=(12\sigma^2n)
      \mathbb{P}(G_3^c)^{1/q}
      +q\sigma^2\sqrt{n}
      \mathbb{P}(\sum_{i=1}^n\varepsilon_i^{j2}\ge 12\sigma^2n)^{\frac{1}{2q}}.
\end{align*}
As a result,
\begin{align*}
  \eqref{eq:epsbound}
  \leq
  \sup_{q \geq 1} \frac{1}{q}
  (12\sigma^2n
  + q \sigma^2\sqrt{n}
  \mathbb{P}(\sum_{i=1}^n\varepsilon_i^{j2}\ge 12\sigma^2n)^{\frac{1}{2q}}
  \mathbb{P}(G_3^c)^{-1/q}).
\end{align*}
Note that, by \cite[Corollary of Lemma 1]{10.1214/aos/1015957395},
$ \mathbb{P}(\sum_{i=1}^n\varepsilon_i^{j2}\ge 12\sigma^2n) \le\exp(-3n)$
and
\begin{align*}
  \mathbb{P}(G_3^c)
  =\mathbb{P}\left(\sqrt{\sum_{i=1}^n(\frac{\varepsilon_i^j}{\sigma})^2}\ge \sqrt{2n}\right)
  \ge \mathbb{P}(\frac{1}{\sqrt{n}}\sum_{i=1}^n\frac{\varepsilon_i^j}{\sigma}
  \ge \sqrt{2n})
  \stackrel{(\rm i)}{\ge}\frac{1}{\sqrt{2\pi}}\frac{\sqrt{2n}}{2n+1}\exp(-n)
\end{align*}
where $(\rm i)$ follows from the lower bound of complementary cumulative
distribution function of standard Gaussian $\Phi^c(t)\ge
\frac{1}{\sqrt{2\pi}}\frac{t}{t^2+1}\exp(-t^2/2)$.
Then,
\begin{align*}
  \mathbb{P}(\sum_{i=1}^n\varepsilon_i^{j2}\ge 12\sigma^2n)^{1/2}\mathbb{P}(G_3^c)^{-1}
  &\le\exp(-\frac{3n}{2})\sqrt{2\pi}\frac{2n+1} {\sqrt{2n}}\exp(n)
    = O(\sqrt{n}\exp(-n/2)).
\end{align*}
Thus,
\begin{align*}
  \eqref{eq:epsbound}
  &\leq
    \sup_{q\ge 1}q^{-1}
    (12\sigma^2n+q\sqrt{n}\sigma^2(\sqrt{n}\exp(-n/2))^{1/2})^{1/q})
     =O(\sqrt{n}\sigma^2e^{-n/2}).
\end{align*}
Therefore, the sub-exponential norm of~\eqref{eq:IV.1} is of the order
$n^{-1/4}\sigma^2 e^{-n/2}$.
Then, Lemma~\ref{lemma.8} applies to~\eqref{eq:IV} for
\begin{align*}
  t
  &=
    O\left((\frac{\sigma^2 e^{-n/2}}{n^{1/4}}+\frac{\Delta_{\max}}{\sqrt{n}})
    \sqrt{e^{-n}\vee \frac{\log(dK^2/\delta)}{m}}
    \sqrt{\frac{d\log(dK^2/\delta)}{m}}\right).
\end{align*}

This concludes the analysis of all sub-parts of $\hat{B}$. We are now ready to
put the piece together to obtain the total bound for $\hat{B}$.

\subsubsection*{Conclusion of $\hat{B}$}
Putting all the terms together and taking union over $K$ elements,
we have the following with probability at least $1-3\delta/K$
\begin{align*}
  &\|\hat{B}-B\|
  \\
  &\lesssim
    \sqrt{\frac{d\log(dK^2/\delta)}{m}}
    \left(
    \frac{\sigma}{\sqrt{n}} \left( \frac{1}{K} \vee \frac{\log(dK^2/\delta)}{m} \right)^{1/2}
    +
    \frac{K(\sigma +\Delta_{\max}) e^{-n}}{\sqrt{n}} \left( \frac{1}{K} \vee
    \frac{\log(dK^2/\delta)}{m} \right)^{1/2}
    \right.
    \\
  &
    \qquad \qquad \qquad
    \left.
    +
    K\frac{\sigma}{\sqrt{n}}
    \left( e^{-n/16} \vee \frac{\log(dK^2/\delta)}{m} \right)^{1/2}
    +
    K(\frac{D_M +\sigma}{n^{1/4}}) \left( 2e^{-C_{\alpha} n} \vee
    \frac{\log(dK^2/\delta)}{m} \right)^{1/2}
    \right.
    \\
  &
    \qquad \qquad \qquad
    \left.
    +
    K\big(\frac{\sigma^2 e^{-n/2}}{n^{1/4}}+\frac{\Delta_{\max}}{\sqrt{n}}\big)
    \left( e^{-n} \vee \frac{\log(dK^2/\delta)}{m} \right)^{1/2}
    \right).
\end{align*}
\subsection*{Bounding $\hat{A}$}
Note that for this term,
\[
  \frac{1}{mn}\sum_{j=1}^mw_1^j(\widehat\btheta)\sum_{i=1}^n\bX_i^j\bX_i^{jT}
  \geq
  \frac{1}{mn}\sum_{j=1}^m
  w_1(\widehat\btheta)\sum_{i=1}^n\bX_i^j\bX_i^{jT}\mathds{1}(Z_j = 1).
\]
Thus, it is sufficient
to bound the deviation of the expression conditional on the event $\{Z_j = 1\}$
from its expectation.
Using $p = 1/K$ and subexponential norm of $n^{-1/2}$ for the conditional sample
average of the outer product of $\bX_i$ over $n$, we apply
Lemma~\ref{lemma.8} for
\begin{align*}
  t \asymp
  \sqrt{\frac{1}{K} \vee \frac{\log(dK^2/\delta)}{m}}\sqrt{\frac{d \log(dK^2/\delta)}{mn}}.
\end{align*}
Thus, by lemma~\ref{lemma.8},
\begin{align*}
  \|
  \frac{1}{mn}\sum_{j=1}^m
  w_1(\widehat\btheta)\sum_{i=1}^n\bX_i^j\bX_i^{jT}\mathds{1}(Z_j = 1)
  &-
    \frac{1}{n}\E[w_1(\widehat\btheta)\sum_{i=1}^n\bX_i^j\bX_i^{jT}\mathds{1}(Z_j = 1)]
    \|
  \\
  &
    \le
    \Bigg(\frac{1}{K} \vee \frac{\log(dK^2/\delta)}{m}\Bigg)^{1/2}\Bigg(\frac{d \log(dK^2/\delta)}{mn}\Bigg)^{1/2}
    ,
\end{align*}
with probability at least $1- 3\delta/K^2$.
Furthermore, we know from~\eqref{eq:A_bound} that
\begin{align*}
  \|\frac{1}{n}\E[w_1(\widehat\btheta)\sum_{i=1}^n\bX_i^j\bX_i^{jT}\mathds{1}(Z_j = 1)]\|
  \geq \frac{1-(K-1)e^{-n}}{K}.
\end{align*}
As a result,
\begin{align*}
  \|\frac{1}{mn}\sum_{j=1}^mw_1^j(\widehat\btheta)\sum_{i=1}^n\bX_i^j\bX_i^{jT}\|
  \geq
  \frac{1-(K-1)e^{-n}}{K}
  +
  \Bigg(\frac{1}{K} \vee \frac{\log(dK^2/\delta)}{m}\Bigg)^{1/2}\Bigg(\frac{d \log(dK^2/\delta)}{mn}\Bigg)^{1/2},
\end{align*}
which implies that $\|\hat{A}\|^{-1} \leq K$.

\subsection*{Final bound}
Recall that we broke down the bounding exercise as
\begin{align*}
  \|\hat{A}\|^{-1}\|\hat{B}\|
  \leq
  \|\hat{A}\|^{-1} \|\hat{B}- B\|
  +
  \|\hat{A}\|^{-1} \|B\|
\end{align*}
Thus, up to log-terms,
\begin{align*}
    \|\hat{A}\|^{-1}\|\hat{B} - B\|
  &
    \lesssim
    K\sqrt{\frac{d}{m}}
    \left(
    \frac{\sigma}{\sqrt{nm}}
    +
    \frac{K(\sigma +\Delta_{\max}) e^{-n}}{\sqrt{nm}}
    +
    K\frac{\sigma}{\sqrt{n}}
    \left( e^{-n/16} \vee \frac{1}{m} \right)^{1/2}
    \right.
  \\
  &
    \qquad \qquad \qquad \qquad \qquad
    \left.
    +
    K(\frac{D_M +\sigma}{n^{1/4}}) \left( 2e^{-C_{\alpha} n} \vee
    \frac{1}{m} \right)^{1/2}
    \right.
  \\
  &
    \qquad \qquad \qquad \qquad \qquad
    \left.
    +
    K\big(\frac{\sigma^2 e^{-n/2}}{n^{1/4}}+\frac{\Delta_{\max}}{\sqrt{n}}\big)
    \left( e^{-n} \vee
    \frac{1}{m} \right)^{1/2}
    \right),
\end{align*}
and
\begin{align*}
  \|\hat{A}\|^{-1}\|B\|
  &
    \lesssim
    (n(\sigma+\Delta_{\max})+K^2\sigma\sqrt{n})e^{-n}
    +(\alpha n \Delta_{\min} + K^2)\sigma\sqrt{n}e^{-C_\alpha n}
  \\
  & \qquad
    +(\sqrt{n}\Delta_{\max} + K^2\sigma\sqrt{n})e^{-n/K^2},
\end{align*}
where we used the assumption that
$m \geq K \log(dK^2/\delta)$ to simplify some of the terms.
It should be clear from the two bounds above that $\|\hat{A}\|^{-1}\|B\|$ is
always of higher-order compared to $\|\hat{A}\|^{-1}\|\hat{B} - B\|$.

As a result the leading order of convergence of the empirical EM depends on the
relationship between $m$ and $n$.
In particular, the rates have a shift at the point when $m \asymp e^{-n}$.
If $m \lesssim e^{n}$,
\begin{align*}
  \|\hat{A}\|^{-1}\|\hat{B}\|
  &
    \lesssim
    \frac{K^2 \sqrt{d}(D_M +\sigma)}{mn^{1/4}}
    +\frac{K^2\sqrt{d}\Delta_{\max} +K(K+1)\sqrt{d}\sigma}{m\sqrt{n}}
  \\
  &
    + e^{-n}
    \left(
    \frac{K^2\sqrt{d}(\sigma +\Delta_{\max})}{m\sqrt{n}}
    +
    \frac{\sqrt{d}K\sigma^2}{mn^{1/4}}
    +
    n(\sigma+\Delta_{\max})
    \right.
  \\
  &
    \qquad \qquad
    \left.
    +\sigma\sqrt{n}(3K^2 + \alpha n \Delta_{\min})
    +\sqrt{n}\Delta_{\max}
    \right).
\end{align*}
We note here that the first two terms are the leading rates of converence.
Assuming $d$, $K$ and $\sigma$ are constants, the leading terms can be
simplified to
\begin{align*}
  \|\hat{A}\|^{-1}\|\hat{B}\|
  &
    \lesssim
    \frac{D_M}{mn^{1/4}}
    +\frac{\Delta_{\max}}{m\sqrt{n}}
    + O((n^{3/2}\Delta_{\min} + n\Delta_{\max})e^{-n}).
\end{align*}
On the other hand, if $m \gtrsim e^{n}$, the rate of convergence is
\begin{align*}
  \|\hat{A}\|^{-1}\|\hat{B}\|
  &
    \leq
    K\sqrt{d}e^{-n/2}\left( \frac{\sigma}{\sqrt{n}}e^{-n/2} + \frac{KD_M}{n^{1/4}}e^{-C_\alpha n/2} +  O(\frac{e^{-n}}{n^{1/4}}) \right)
  \\
  &
    \leq
    K\sigma\sqrt{\frac{d}{n}}e^{-n} + \frac{KD_M}{n^{1/4}}e^{-(C_\alpha-1) n/2} +  O(\frac{e^{-n}}{n^{1/4}})
\end{align*}
where the dependency on all other parameters are sewpt into in the Big-O term.

\end{proof}

\subsection{Proof of Corollary~\ref{cor.1}}
\begin{proof}
Define $D_M^{(t)}: = \max_{k\in [K]}\|\widehat\btheta^{(t)}_k-\btheta_k^*\|$.
We can assume $D_M^{(t)}> \varepsilon$
$\forall t = 0, 1,\dots, T-1$ since otherwise the result follows trivially.
We start by proving the first statement of the theorem, under the assumption
that $m < \exp(-n)$.
Note that by Theorem~\ref{thm:emp_consistency},
\begin{align*}
  D_M^{(t)}
  \leq
  \frac{D_M^{(t-1)}}{mn^{1/4}}
  +\frac{\Delta_{\max}}{m\sqrt{n}}
  + (n^{3/2}\Delta_{\min} + n\Delta_{\max})e^{-n}.
\end{align*}
This gives us a recursive equation that can be solved as follows:
\begin{align}
  \label{eq:recursion}
  D_M^{(T)}
  \leq
  \frac{D_M^{(0)}}{(mn^{1/4})^T}
  +(\frac{\Delta_{\max}}{m\sqrt{n}}
  + (n^{3/2}\Delta_{\min} + n\Delta_{\max})e^{-n})
  \sum_{j=1}^T \frac{1}{(mn^{1/4})^j}.
\end{align}
Then, by the assumption that
$\frac{\Delta_{\max}}{m\sqrt{n}} + (n^{3/2}\Delta_{\min} + n\Delta_{\max})e^{-n} \leq \varepsilon/2$,
we need solve for $T$ such that we can guarantee that the first term on the RHS
of \eqref{eq:recursion} is bounded by $\varepsilon/2$.
Simple algebraic manipulation shows that for
$T\geq \frac{2\log(\frac{\alpha \Delta_{\min}}{\varepsilon})} {\log(mn^{1/4})}$
the desired control on the maximum error is achieved.

For the second statement, we repeat the exercise of first setting up the
resursion:
\begin{align*}
  D_M^{(t)}
  \leq
  D_M^{(t-1)} \frac{Ke^{-(C_\alpha-1) n/2}}{n^{1/4}}
  + K\sigma\sqrt{\frac{d}{n}}e^{-n}
  + O\big(\frac{e^{-n}}{n^{1/4}}\big),
\end{align*}
which can be solved for
\begin{align*}
  D_M^{(T)}
  \leq
  D_M^{(0)}
  \left( \frac{Ke^{-(C_\alpha-1) n/2}}{n^{1/4}} \right)^T
  + (K\sigma\sqrt{\frac{d}{n}}e^{-n} + \frac{e^{-n}}{n^{1/4}})
  \sum_{j=0}^{T}\left( \frac{Ke^{-(C_\alpha-1) n/2}}{n^{1/4}} \right)^j,
\end{align*}
which, if
$e^{-n}(K\sigma\sqrt{\frac{d}{n}} + n^{-1/4}) \leq \varepsilon/2$, then for any
\begin{align*}
  T\geq \frac{\log(\frac{2\alpha\Delta_{\min}}{\varepsilon})}{n +\frac{1}{4}\log n - \log K},
\end{align*}
the empirical loss is guaranteed to be at most $\varepsilon$.
We observe that in this case for $n$ sufficiently large, we only need a constant
number of iterations to achieve convergence.

\end{proof}

\section{Auxiliary Lemmas}\label{appendix:technical}
\begin{lemma}\label{lemma.8}
  Let $K$ be the number of components in the
  FMLR. Let $U$ be a $d$-dimensional random variable and $A$ be an event defined
  on the same probability space with $p = \mathbb{P}(U\in A)\le\frac{1}{K}$. Define the
  random variables $W = U|A$ and $Z = \mathds{1}_{A}$. Suppose $W$ is sub-exponential
  with sub-exponential norm $\|W\|_{\psi_1}$.
  Let $U_j, W_j, Z_j$ be the i.i.d
  samples from the corresponding distributions. Then, for
  \begin{align*}
    t \asymp \|W\|_{\psi_1}\sqrt{p\vee
    \frac{\log(dK^2/\delta)}{m}}\sqrt{\frac{d\log(dK^2/\delta)}{m}},
  \end{align*}
  with probability at least $1-3\delta/K^2$, we have
    $$\|\frac{1}{m}\sum_{j=1}^mU_jZ_j-\mathbb{E}[UZ]\|\le t.$$
  \end{lemma}
\begin{proof}
The key idea of the proof lies in the application of~\cite[Proposition
5.3]{kwon2020converges} which controls the deviation of a conditional sample
average from
its expectation.
We start by defining $Z_j = \mathds{1}_{U_j\in A}$ and $p=\mathbb{P}(A)$.
Then, observe that $Z_j$ is a Bernoulli random variable with $p$.
By Bernstein's inequality for Bernoulli random variables,
\[\mathbb{P}(|\frac{1}{m}\sum_{j=1}^mZ_j-p|\ge s)\le\exp(-\frac{ms^2}{2p+\frac{2}{3}s}).\]
To identify the right threshold (like in \cite[Proposition 5.3]{kwon2020converges}), we want to guarantee
\[\mathbb{P}(\sum_{j=1}^mZ_j\ge m_e+1)\le\mathbb{P}(\frac{1}{m}\sum_{j=1}^mZ_j-\mathbb{E}[z]\ge s)\le\frac{\delta}{K^2}.\]
Therefore, we choose
\begin{align*}
  s =
  \frac{1}{m}
  \left(
  \frac{1}{3}\log(\frac{K^2}{\delta})
  +\left(\frac{1}{9}\log^2(\frac{K^2}{\delta})
  +2pm\log(\frac{K^2}{\delta}) \right)^{1/2} \right)
\end{align*}
and
\begin{align*}
    m_e=mp+ms
    &=mp+\mathcal{O}(\log(K^2/\delta)\vee\sqrt{pm\log(K^2/\delta)}).
\end{align*}
Note that since $p\le \frac{1}{K}$ and $m\ge \Omega(K)$, $m_e\le m$.
Now, using the fact that $\mathbb{E}[\|W\|]\le\|W\|_{\psi_1}$,
by Bernstein's inequality, for $t_2$ in \cite[proposition 5.3]{kwon2020converges},
\begin{align*}
    t_2
  &\lesssim
    \|W\|_{\psi_1}\sqrt{p\vee\frac{\log(K^2/\delta)}{m}}\sqrt{\frac{\log(K^2/\delta)}{m}}.
\end{align*}
Next, since we assume $W$ is sub-exponential,
by \cite[Theorem 2.8.2]{vershynin2018high}
\begin{align*}
  \mathbb{P}(|\frac{1}{m}\sum_{j=1}^{\tilde m}W_j-\mathbb{E}[W]|\ge t_1)
  \leq\exp(-C\min\Big\{\frac{mt_1}{\|W\|_{\psi_1}\sqrt{d}}, \frac{m^2t_1^2}{m_ed\|W\|_{\psi_1}^2}\Big\}+C'\log d),
\end{align*}
for all $\tilde{m}\leq m_e$.
Therefore,
\begin{align*}
  t_1 \asymp\|W\|_{\psi_1}
  \sqrt{p\vee \frac{\log(dK^2/\delta)}{m}}
  \sqrt{\frac{d\log(dK^2/\delta)}{m}}.
\end{align*}
Plugging in each of these terms into the statement of \cite[proposition
5.3]{kwon2020converges} concludes the proof.
\end{proof}

The following two lemmas are used in bounding sub-exponential norms of random
variables conditioning on some events.
Note that these statements are similar to Lemma A.1 and Lemma A.2 in
\cite{kwon2020converges} with the caveat that \cite{kwon2020converges} focuses on $\langle X,
u\rangle$, while the following lemmas address the case of $\langle X, u\rangle^2$.
\begin{lemma}\label{lemma.1}
  Let $\bX_1,\dots, \bX_n\stackrel{\text{i.i.d.}}{\sim}\mathcal{N}(0, I_d)$. For any fixed vector
    $u$ and constant $\alpha$, define $G = \{\sum_{i=1}^n\langle X_i, u\rangle^2\ge
    \alpha^2\}$. Then for any unit vector $s\in \mathcal{S}^{d-1}$ and $p\ge 1$,
    \[\mathbb{E}[(\sum_{i=1}^n\langle \bX_i, s\rangle ^2)^p|G^c] = O((\sqrt n p)^p).\]
\end{lemma}
\begin{proof}
    Without loss of generality, we can assume $u = e_1$ due to the rotational invariance property of
Gaussian. Denote $Y_i = \langle \bX_{i, 2:d}, s_{2:d}\rangle$ as the inner product
between the second to the last coordinates of $\bX_i$ and $s$. Then we have
\begin{align*}
  &
    \mathbb{E}[(\sum_{i=1}^n\langle \bX_i, s\rangle
    ^2)^p|G^c]=\frac{\mathbb{E}[(\sum_{i=1}^n(s_1\bX_{i,1}+Y_i)^2)^p
    \mathds{1}_{\sum_{i=1}^n\bX^2_{i,1}\le\alpha^2}]}{\mathbb{P}(\sum_{i=1}^n\bX^2_{i,1}\le\alpha^2)}
  \\
  &\le \frac{\mathbb{E}[(\sum_{i=1^n}2s^2_1\bX_{i, 1}^2+2Y_i^2)^p
    \mathds{1}_{\sum_{i=1}^n\bX_{i, 1}^2\le\alpha^2}]}
    {\mathbb{P}(\sum_{i=1}^n\bX^2_{i,1}\le\alpha^2)}
  \\
  &
    =\frac{\mathbb{E}[(\mathbb{E}[(\sum_{i=1^n}2s^2_1\bX_{i,1}^2+2Y_i^2)^p|\{\bX_{i,1}\}_{i=1}^n]^{1/p})^p
    \mathds{1}_{\sum_{i=1}^n\bX_{i, 1}^2\le\alpha^2}]}{\mathbb{P}(\sum_{i=1}^n\bX^2_{i,1}\le\alpha^2)}
  \\
  &\stackrel{(\rm i)}{\leq}
    \frac{\mathbb{E}[(\mathbb{E}[(\sum_{i=1^n}2s^2_1\bX_{i,1}^2)^p|\bX_{i, 1}\}_{i=1}^n]^{1/p}
    +\mathbb{E}[(\sum_{i=1}^n2Y_i^2)^p|\{\bX_{i,1}\}_{i=1}^n]^{1/p})^p
    \mathds{1}_{\sum \bX_{i, 1}^2\le\alpha^2}]}{\mathbb{P}(\sum_{i=1}^n\bX^2_{i,1}\le\alpha^2)}
  \\
  &\stackrel{(\rm ii)}{=}
    \frac{\mathbb{E}[(\sum_{i=1}^n2s_1^2\bX_{i,1}^2
    +\mathbb{E}[(\sum_{i=1}^n2Y_i^2)^p]^{1/p})^p
    \mathds{1}_{\sum_{i=1}^n \bX_{i, 1}^2\le\alpha^2}]}{\mathbb{P}(\sum_{i=1}^n\bX^2_{i,1}\le\alpha^2)}\\
  &\stackrel{(\rm iii)}{\leq}
    \frac{(2s_1^2\alpha^2+\mathbb{E}[(\sum_{i=1}^n2Y_i^2)^p]^{1/p})^p
    \mathbb{E}[\mathds{1}_{\sum_{i=1}^n\bX_{i,1}^2\le \alpha^2}]}{\mathbb{P}(\sum_{i=1}^n\bX^2_{i,1}\le\alpha^2)}
  \\
  &\stackrel{(\rm iv)}{=}
    (2s_1^2\alpha^2+C\sqrt{n}\|s_{2:d}\|^2p)^p
    =O((\sqrt n p)^p).
\end{align*}
Note that $(\rm i)$ follows from Minkowski inequality, both $(\rm ii)$ and
$(\rm iii)$ follow from the independence of $\{\bX_{i,1}\}_{i=1}^n$ and
$\{Y_{i}\}_{i=1}^n$, and $(\rm iv)$ follows as $\sum_{i=1}^n2Y_i^2\sim
\operatorname{SubExp}(16n\|s_{2:d}\|^4, 8\|s_{2:d}\|^2)$ whose $L_p$ norm is
$C\sqrt{n}\|s_{2:d}\|^2p$ for some positive constant $C$.
\end{proof}

\begin{lemma}\label{lemma.2}
  Let $\bX_1, \dots, \bX_n\stackrel{\text{i.i.d.}}{\sim}\mathcal{N}(0, I_d)$. For any fixed vector
$\bu\in \mathbb{R}^d$ and a set of vectors $\{\bv_1,\dots, \bv_H\}\subset\mathbb{R}^d$
such that $\|\bu\|\ge \|\bv_l\|$ $\forall l=1,\dots, H$, define
$G:=\cap_{l=1}^H\{\sum_{i=1}^n\langle \bX_i, \bu\rangle^2 \ge \sum_{i=1}^n\langle
\bX_i, \bv_l\rangle ^2\}$. Then for any unit vector $s\in\mathcal{S}^{d-1}$ and
$p\ge 1$,
\begin{align*}
  \mathbb{E}[(\sum_{i=1}^n\langle \bX_i, s\rangle^2)^p|G^c] = O(H(np)^p).
\end{align*}
\end{lemma}
\begin{proof}
Let $G_l =\{\sum_{i=1}^n\langle \bX_i, \bu\rangle ^2\ge \sum_{i=1}^n\langle \bX_i,
\bv_l\rangle^2\}$. Then $G = \cap_{l=1}^HG_l$. We first focus on $G_1^c$.
By the rotational invariance property of Gaussian, we can assume
$\text{span}\{\bu, \bv_1\} = \text{span}\{\be_1, \be_2\}$, where $\be_i$ is the
$i$-\textit{th} standard basis vector. We use the
following change of coordinates $\bX_{i, 1} = r_i\cos\theta_i$ and
$\bX_{i,2} = r_i\sin\theta_i$ where
$r_i\stackrel{\text{i.i.d.}}{\sim}\text{Rayleigh}(1)$ and
$\theta_i\stackrel{\text{i.i.d.}}{\sim}\text{Uniform}[0, 2\pi]$.
Define $Y_i = \langle \bX_{i, 3:d}, s_{3:d}\rangle$.
\begin{align*}
  &
  \mathbb{E}[(\sum_{i=1}^n\langle \bX_i, s\rangle ^2)^p|G_1^c]
    \\
  &=\frac{\mathbb{E}[(\sum_{i=1}^n(s_1r_i\cos\theta_i+s_2r_i\sin\theta_i+Y_i)^2)^p\mathds{1}_{G_1^c}]}{\mathbb{P}(G_1^c)}
  \\
  &=\frac{\mathbb{E}_{\theta}[(\mathbb{E}_{r, Y}[(\sum_{i=1}^n(s_1r_i\cos\theta_i+s_2r_i\sin\theta_i+Y_i)^2)^p|\theta]^{1/p})^p\mathds{1}_{G_1^c}]}{\mathbb{P}(G_1^c)}
  \\
  &\stackrel{(\rm i)}{\leq}
    \frac{\mathbb{E}_{\theta}[(\mathbb{E}_{r, Y}
    [(\sum_{i=1}^n4r_i^2(s_1^2\cos^2\theta_i+s_2^2\sin^2\theta_i)+\sum_{i=1}^n2Y_i^2)^p|\theta]^{1/p})^p\mathds{1}_{G_1^c}]}{\mathbb{P}(G_1^c)}
  \\
  &\stackrel{(\rm ii)}{\leq}
    \frac{\mathbb{E}_{\theta}[(\mathbb{E}_r
    [(\sum_{i=1}^n4r_i^2(s_1^2\cos^2\theta_i+s_2^2\sin^2\theta_1))^p|\theta]^{1/p}
    +\mathbb{E}_Y[(\sum_{i=1}^n2Y_i^2)^p]^{1/p})^p\mathds{1}_{G_1^c}]}
    {\mathbb{P}(G_1^c)}
\end{align*}
where $(\rm i)$ follows from the inequality $(a+b)^2\le 2a^2 + 2b^2$ and $(\rm
ii)$ follows from Minkowski inequality.
Note that
$\sum_{i=1}^n2Y_i^2\sim\text{SubE}(16n\|s_{3:d}\|^4, 8\|s_{3:d}\|^2)$
whose $L_p$ norm is $C\sqrt{n}\|s_{3:d}\|^2p$ for some constant $C$.
Moreover,
\begin{align*}
  &
    \mathbb{E}_r[(\sum_{i=1}^n4r_i^2(s_1^2\cos^2\theta_i+s_2^2\sin^2\theta_1))^p|\theta]^{1/p}
  \\
  &\leq
    \mathbb{E}_r[(\sum_{i=1}^n16r_i^4)^{p/2}]^{1/p}(\sum_{i=1}^n(s_1^2\cos^2\theta_i+s_2^2\sin^2\theta_i)^2)^{1/2}
  \\
  &
    \leq \mathbb{E}_r[(4\sqrt{n}r^2)^p]^{1/p}\sqrt{n}\|s_{1:2}\|^2
  \\
  &
    =4n\|s_{1:2}\|^2\mathbb{E}_r[r^{2p}]^{1/p},
\end{align*}
where the first inequality follows by Cauchy-Schwarz inequality.
Therefore,
\begin{align*}
  \mathbb{E}[(\sum_{i=1}^n\langle \bX_i, s\rangle ^2)^p|G_1^c]
  &\leq
    \frac{(4n\|s_{1:2}\|^2\mathbb{E}_r[r^{2p}]^{1/p}+C\sqrt{n}\|s_{3:d}\|^2p)^p
    \mathbb{E}_{\theta}[\mathds{1}_{\theta\in G_1^c}]}{\mathbb{P}(G_1^c)}
  \\
  &=(4n\|s_{1:2}\|^2\mathbb{E}_r[r^{2p}]^{1/p}+C\sqrt{n}\|s_{3:d}\|^2p)^p.
\end{align*}
Since $r\sim\text{Rayleigh}(1)$, its raw moments are given by
$2^{p/2}\Gamma(1+\frac{p}{2})$ where $\Gamma$ is the Gamma function. Then,
\begin{align*}
  \mathbb{E}_r[r^{2p}]^{1/p}
  = (\mathbb{E}_r[r^{2p}]^{\frac{1}{2p}})^2
  = 2\Gamma^{1/p}(1+p).
\end{align*}
Note that by Lanczos approximation, $\Gamma^{1/p}(1+p) = O(p)$. This gives us
\begin{align*}
  \mathbb{E}[(\sum_{i=1}^n\langle \bX_i, s\rangle ^2)^p|G_1^c]
  \leq (8n\|s_{1:2}\|^2\Gamma^{1/p}(1+p)+C\sqrt n\|s_{3:d}\|^2p)^p
  =O((np)^p).
\end{align*}
Replicating the analysis for all $G_l^{c}$ for $l = 2, \ldots, H$,
\begin{align*}
  \mathbb{E}[(\sum_{i=1}^n\langle \bX_i, s\rangle ^2)^p|G^c]
   &
     \leq
     \frac{\mathbb{E}[(\sum_{i=1}^n\langle \bX_i, s\rangle^2)^p
     \sum_{l=1}^H\mathds{1}_{G_l^c}]}
    {\mathbb{P}(G^c)}
  \\
  &
    \leq \sum_{l=1}^H
    \frac{\mathbb{E}[(\sum_{i=1}^n\langle \bX_i, s\rangle^2)^p\mathds{1}_{G_l^c}]}
    {\mathbb{P}(G^c_l)}
    =O(H(np)^p)
\end{align*}
\end{proof}

\section{Experiment Details}\label{experiment details}
For the purpose of replicability, we report ground truth cluster centers that we
used in the experiments in Section~\ref{sec:experiments}.

\subsubsection*{Figure \ref{fig:experiments_with_n}}
Set $K = 3$, $d = 5$,\\
$\btheta_1 = 3\times\mathds{1}_{\mathbb{R}^5}$,\\
$\btheta_2 = 0$ and \\
$\btheta_3 = -3\times\mathds{1}_{\mathbb{R}^5}$.

\subsubsection*{Figure \ref{fig:exp_with_K}}
Set $n = 5$ and $d = 2$.
We choose the following centers based on $K$, while maintaining an SNR of
approximately 28.\\
For $K = 2$: $\btheta_1 = [10, 10]$ and $\btheta_2 = [-10, -10]$.\\
For $K = 4$:
\begin{align*}
  \btheta_1 &= [-14, 14],  &&\btheta_2 = [14, 14],\\
  \btheta_3 &= [-14, -14],  &&\btheta_4= [14, -14].
\end{align*}
For $K = 6$:
\begin{align*}
  \btheta_1 &= [-14,24],&& \btheta_2 = [14,24],&&& \btheta_3 = [28, 0],\\
  \btheta_4 &= [14, -24], &&\btheta_5 = [-14, -24] &&&\btheta_6 = [-28, 0]
\end{align*}
For $K = 8$:
\begin{align*}
  \btheta_1 &= [-14, 34], &&\btheta_2 = [14, 34],  &&&\btheta_3 = [34, 14], &&&& \btheta_4= [34, -14],\\
  \btheta_5 &= [14, -34], &&\btheta_6 = [-14, -34], &&& \btheta_7= [-34, -14], &&&&\btheta_8 = [-34, 14].
\end{align*}

\subsubsection*{Figure \ref{fig:exp_with_d}}
Set $n = 5$ and $K = 2$.
Choosing $\btheta_2 = -\btheta_1$, while maintaining an SNR of approxmately 28.
For $d = 2$: $\btheta_1 = 10\times\mathds{1}_{\mathbb{R}^2}$.\\
For $d = 4$: $\btheta_1 = 7\times\mathds{1}_{\mathbb{R}^4}$.\\
For $d = 6$: $\btheta_1 =6\times \mathds{1}_{\mathbb{R}^6}$.\\
For $d = 8$: $\btheta_1 = 5\times \mathds{1}_{\mathbb{R}^8}$.\\

\subsubsection*{Figure \ref{fig:exp_with_SNR}}
Set $n = 3$, $d = 3$ and $K = 3$.
Choosing $\btheta_2 = -\btheta_1$ and $\btheta_3 = 0$.\\
For $\text{SNR} = 0.87$: $\btheta_1 =
\frac{1}{2}\times\mathds{1}_{\mathbb{R}^3}$.\\
For $\text{SNR} = 1.73$:
$\btheta_1 = \mathds{1}_{\mathbb{R}^3}$.\\
For $\text{SNR} = 6.93$: $\btheta_1 = 4\times\mathds{1}_{\mathbb{R}^3}$.\\
For $\text{SNR} = 13.86$: $\btheta_1 = 8\times\mathds{1}_{\mathbb{R}^3}$.

\subsubsection*{Figure \ref{fig:exp_with_max}}
Set $n = 5, d = 3$ and $K = 3$.
Choosing $\btheta_1 = \mathds{1}_{\mathbb{R}^3}$ and $\btheta_2 =
-\mathds{1}_{\mathbb{R}^3}$ to ensure the SNR remains constant.\\
For $\Delta_{\max} = 19$: $\btheta_3 = 10\times\mathds{1}_{\mathbb{R}^3}$.\\
For $\Delta_{\max} = 54$: we set $\btheta_3 =
30\times\mathds{1}_{\mathbb{R}^3}$.\\
For $\Delta_{\max} = 105$: $\btheta_3 = 60\times\mathds{1}_{\mathbb{R}^3}$.\\
For $\Delta_{\max} = 209$: $\btheta_3 = 120\times\mathds{1}_{\mathbb{R}^3}$.

\end{appendix}

\bibliographystyle{imsart-nameyear} 
\bibliography{myref.bib}       
\end{document}